\documentclass[a4paper]{article}

\usepackage[a4paper,hmargin=1.5cm]{geometry}

\usepackage{times}
\usepackage[utf8]{inputenc}
\usepackage[english]{babel}
\usepackage[pagebackref=true,breaklinks=true,letterpaper=true,colorlinks,bookmarks=false]{hyperref}
\usepackage{epsfig}
\usepackage{epstopdf}
\usepackage{graphicx}
\usepackage{amsmath}
\usepackage{amsthm}
\usepackage{amssymb}
\usepackage{ mathrsfs }
\usepackage{color}
\usepackage{algpseudocode}
\usepackage{algorithm}
\usepackage{placeins}
\usepackage{rotating}
\usepackage{capt-of,etoolbox}
\usepackage{rotating}
\usepackage[framemethod=TikZ]{mdframed}
\usepackage{lipsum}
\mdfdefinestyle{MyFrame}{%
    linecolor=blue,
    outerlinewidth=2pt,
    roundcorner=20pt,
    innertopmargin=\baselineskip,
    innerbottommargin=\baselineskip,
    innerrightmargin=20pt,
    innerleftmargin=20pt,
    backgroundcolor=gray!50!white}

\newcommand\blfootnote[1]{%
  \begingroup
  \renewcommand\thefootnote{}\footnote{#1}%
  \addtocounter{footnote}{-1}%
  \endgroup
}

\usepackage{ulem} 

\newcommand{\argmin}{\operatornamewithlimits{argmin}}
\DeclareMathOperator{\tr}{tr}
\DeclareMathOperator{\prox}{prox}

\DeclareMathOperator{\sign}{sgn}
\newcommand{\G}{\mathcal{G}}
\newcommand{\K}{\mathcal{K}}
\newcommand{\V}{\mathcal{V}}
\newcommand{\E}{\mathcal{E}}
\newcommand{\W}{\mathcal{W}}
\newcommand{\Rbb}{\mathbb{R}}
\newcommand{\Larg}{\mathcal{L}}

\newcommand{\nauman}[1]{{\textcolor[rgb]{0,0,1}{#1}}}

\newtheorem{thm}{Theorem}
\newtheorem{defn}[thm]{Definition}

\title{Low-Rank Matrices on Graphs: Generalized Recovery \& Applications}
\author{Nauman Shahid*, Nathanael Perraudin,  Pierre Vandergheynst \\
LTS2, EPFL, Switzerland \\
*nauman.shahid@epfl.ch}

\begin{document}

\maketitle

\begin{abstract}
\blfootnote{N.Shahid and N.Perraudin are supported by the SNF grant no. 200021\_154350/1 for the project ``Towards signal processing on graphs''.}
Many real world datasets subsume a linear or non-linear low-rank structure in a very low-dimensional space. Unfortunately, one often has very little or no information about the geometry of the space, resulting in a highly under-determined recovery problem. Under certain circumstances, state-of-the-art algorithms provide an exact recovery for linear low-rank structures but at the expense of highly inscalable algorithms which use nuclear norm. However, the case of non-linear structures remains unresolved.   We revisit the problem of low-rank recovery from a totally different perspective, involving graphs which encode pairwise similarity between the data samples and features. Surprisingly, our analysis confirms that it is possible to recover many approximate linear and non-linear low-rank structures with recovery guarantees with a set of highly scalable and efficient algorithms. We call such data matrices as \textit{Low-Rank matrices on graphs} and show that many real world datasets satisfy this assumption approximately due to underlying stationarity. Our detailed theoretical and experimental analysis unveils the power of the simple, yet very novel recovery framework \textit{Fast Robust PCA on Graphs}. 
\end{abstract}

\section{Introduction}

\begin{mdframed}[style=MyFrame]
\begin{center}
\textbf{ \textit{\nauman{``This techncial report is a detailed explanation of the novel low-rank recovery concepts introduced in the previous work, \textit{Fast Robust PCA on Graphs} \cite{shahid2015fast}. The readers can refer to \cite{shahid2015fast} for experiments which are not repeated here for brevity.''}}}. 
\end{center}
\end{mdframed}


Consider some data sampled from a manifold (a low dimensional structure embedded in high dimensional space, like a hyper-surface.) In the first row of Fig.~\ref{fig:demo}, we present some examples of $1$ or $2$ dimensional manifolds embedded in a $2$ or $3$ dimensional space. Given a noisy version of the samples in the space, we want to recover the underlying clean low-dimensional manifold. There exist two well-known approaches to solve this problem. 

\begin{enumerate}
\item One common idea is to assume that the manifold is linear and  the data is low rank. This approach is used in all the Principal Component Analysis (PCA) techniques. Nevertheless, for many datasets, this assumption is only approximately or locally true, see Fig.~\ref{fig:demo} for example. PCA has a nice interpretation in the original data space (normally termed as low-rank representation) as well as in the low-dimensional space (embedding or the projected space). Throughout, we refer to the former as the `data space' and latter as `projected space'.
\item Another idea is to use the manifold learning techniques such as LLE, laplacian eigenmaps or MDS. However, these techniques project the points into a new low dimensional space (projected space), which is not the task we would like to perform. We want the resulting points to remain in the data space. 
\end{enumerate}  

Thus, we need a  framework  to recover low dimensional manifold structure in the original space of the data points. It should be highly scalable to deal with the datasets consisting of millions of samples that  are corrupted with noise and errors.  This paper broadly deals with a class of algorithms which solve such problems, all under one framework, using simple, yet very powerful tools, i.e, graphs. We call our framework \textit{Generalized Recovery of Low-rank Matrices on Graphs} or \textit{Generalized PCA on Graphs (GPCAG)}.

A first reaction is that the goals of this paper are trivial. The problem of low-rank recovery has been very well studied in the last decade. So, why would one still be interested in solving this problem.  Moreover, what is special about specifically  targeting this problem via graphs?  In the following discussion we motivate the utility of this work with several real world examples and gradually describe the tools which are useful for our work.

\subsection{What is a good enough low-rank extraction method?}

\begin{mdframed}[style=MyFrame]
\begin{center}
\textbf{HIGHLIGHT: \textit{\nauman{``The choice of a suitable low-rank recovery method depends on the domain from which the data is acquired. Thus, there is a need to generalize linear subspace recovery to manifolds!''}}}. 
\end{center}
\end{mdframed}

The answer to such a question does not exist. In fact the answer solely depends on the type of data under consideration. Real world datasets are complex, big and corrupted with various types of noise. What makes the low-rank extraction a hard problem is the domain from which the dataset is acquired.  Here we present a few examples.

\subsubsection{Surveillance: Static low-rank}
Consider a real world surveillance application where a camera is installed in an airport lobby and it collects frames with a certain time interval. As the structure of this  lobby does not change over time, one would like to extract the moving people or objects from the static lobby. Such an application clearly falls under the low-rank extraction framework. A good enough assumption about the slowly varying images is that they can be very well defined by a linear subspace in a low-dimensional space. Thus, a linear subspace model, like PCA might work very well for such a dataset.

\subsubsection{Surveillance: Dynamic low-rank}
Now consider the same camera installed outside in the open air where the lightning is changing over time. Assume there are short periods of sunlight interleaved by cloudy periods and even rain. As compared to the lobby example where all these conditions were constant, this is a more challenging environment. It might not be possible to extract the low-rank component of this set of frames via a single linear subspace model. One would need something slightly more complex than this simple assumption. Another example where the linear subspace assumption might fail is the case of a movie where the scenes change over time. A more reasonable assumption would be to target a dynamically changing low-rank representation over short bursts of frames. This would require some sort of grouping between the frames in time.

\subsubsection{Multiclass low-rank}
As another example, consider a huge database of faces of several people collected from a social network platform. Assume further that the faces are corrupted by different types of noise and errors. Furthermore, the faces undergo a change in the facial expressions. A linear subspace model would simply fail to extract a low-rank representation from such a database of faces because it would not take into account the subgroups of the samples in the data. A more reasonable assumption for such a dataset would  be to group the images based on the people using  some notion of pairwise similarity. This would introduce non-linearity in the low-rank extraction method.

\subsubsection{Topic extraction}
Now consider a set of text messages collected from a social network about different topics. One would like to analyze the most common topics in the whole set of messages. A more interesting analysis would be to identify the communities of people in the social network among which a particular topic is more popular. Again, the messages are highly corrupted by noise due to the use of slang and unstructured sentences. In such a case one definitely needs to use the notion of pairwise similarity between the people in the network to identify the low-rank component, i.e, most commonly used words and topics and as well as analyze the community structure. 

\subsubsection{Space-time low-rank}
Finally consider a set of MRI scans which vary slowly over time and the regions of brain. One would like to separate  the clean low-rank components from sparse outliers which might correspond to some kind of abnormal condition. Furthermore, the acquired scans might be corrupted with noise due to limitations of device or the movement of patient, etc. This is a typical example of a space-time real world dataset which has  gained significant attention of the research community. For such a dataset, the notion of functional connectivity between different regions of the brain plays an important role in addition to the structural connectivity. Moreover, the acquired images can be related in time. Thus, it is important to consider the space-time relationship while extracting the low-rank representation. A single linear subspace might not do well in such a scenario.

\subsection{From linear subspaces to manifolds}
In the light of above examples, we can conclude that  a linear subspace model is not always good enough to extract low-rank structures from the data. In the sequel we use the word `manifold' as a very general term to describe  non-linear subspace structures. What type of a manifold assumption is good enough for each of the above applications? As we do not know the answer to this question, we can instead characterize the manifold directly using the data itself.

\begin{figure*}[htbp]
    \centering
        \centering
        \includegraphics[width=1.0\textwidth]{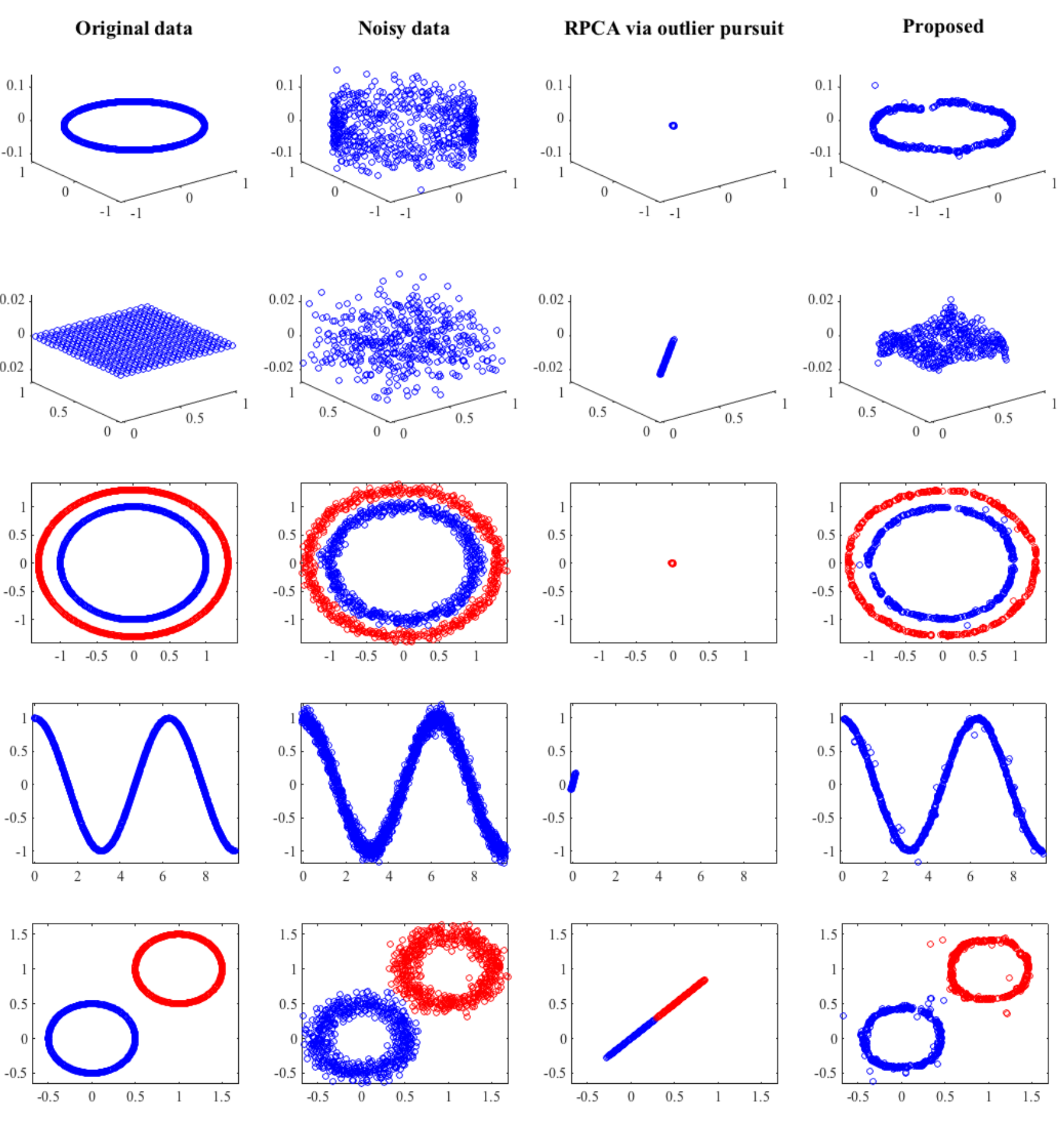}
         \caption{This figure shows the recovery of underlying clean low-rank structure of the noisy 2D or 3D manifolds using the Robust PCA framework \cite{candes2011robust}. RPCA, which is a linear subspace recovery algorithm totally fails in all the scenario as it tries to reduce the whole data to a point at the center of the manifold in an attempt to exactly recover the manifold. The last column shows the approximate desired behavior for each of the cases.}
        \label{fig:demo}
    \end{figure*}

Furthermore, sometimes the variation in the data is huge and we do not have enough samples to extract the low-rank structure. This is a completely reasonable scenario as many real world datasets are highly corrupted by different types of noise and outliers. A general method should be naive enough to simply de-noise the data in such a scenario. Below we provide a list of the types of noise that should be tolerable by any low-rank recovery method:
\begin{itemize}
\item Gaussian noise
\item Sparse and gross errors in a few features of the data.
\item Sample specific errors, where only a few samples of the data are corrupted by errors.
\end{itemize}


\subsection{Motivation: Why a new set of methods?}

\begin{mdframed}[style=MyFrame]
\begin{center}
\textbf{MOTIVATION: \textit{\nauman{``Standard linear low-rank recovery methods, including standard PCA and Robust PCA fail to recover a non-linear manifold. We need new methods to recover low-rank manifolds''}}}. 
\end{center}
\end{mdframed}

We argue with a few simple examples as shown in Fig.~\ref{fig:demo} to demonstrate the motivation of our work. The Fig. shows various datasets which are corrupted by noise. The goal is to:
\begin{enumerate}
\item de-noise and recover the underlying low-rank manifold if the data is low-rank
\item simply de-noise the dataset and recover the clean manifold if it is not low-rank
\end{enumerate} 
We try to perform this task using  the very famous state-of-the-art algorithm Robust PCA (with an $\ell_2$ data fidelity term) \cite{candes2011robust} from the regime of linear dimensionality reduction. This model has the tendency to \textit{exactly} recover the linear low-rank structure of the data. In very simple words the algorithm decomposes  a dataset into a low-rank and a noisy matrix. 

The first two rows of Fig.~\ref{fig:demo} show two different manifolds lying in a 2D-space. We add noise in the third dimension and create an artificial 3D data from these manifolds. The goal here is to de-noise and recover the low-rank manifold in 2D. The Robust PCA (RPCA) framework, which is a linear subspace recovery algorithm totally fails in this scenario.  It can be seen that it tries to reduce the whole data to a point in the center of the manifold. 

Next, consider the datasets shown in the third, fourth and fifth rows of Fig.~\ref{fig:demo}. Here, the manifolds are not low-rank, i.e, they only lie in the 2D space and the noise is introduced in the same dimensions. Apparently, this task is easier than the first one because it only involves de-noising. Again we use RPCA  to recover the manifold. We observe it fails in this task as well. 

The  last column of Fig.~\ref{fig:demo} shows what we would approximately like to achieve for each of these datasets. Clearly RPCA is far from achieving the goal. One can simply argue that RPCA can only extract linear structure from the data so it is supposed to fail in such a scenario. The argument is entirely true. It is possible to use some  non-linear dimensionality reduction algorithm, like the LE framework but it would recover the low rank structure of the manifold \textit{in a projected space}.  Currently we do not have algorithms that would simultaneously de-noise the data, recover the manifold and as well as extract the low-rank structures in the data space itself.

\subsection{Problem statement \& Goals}

\begin{figure*}[htbp]
    \centering
        \centering
        \includegraphics[width=0.9\textwidth]{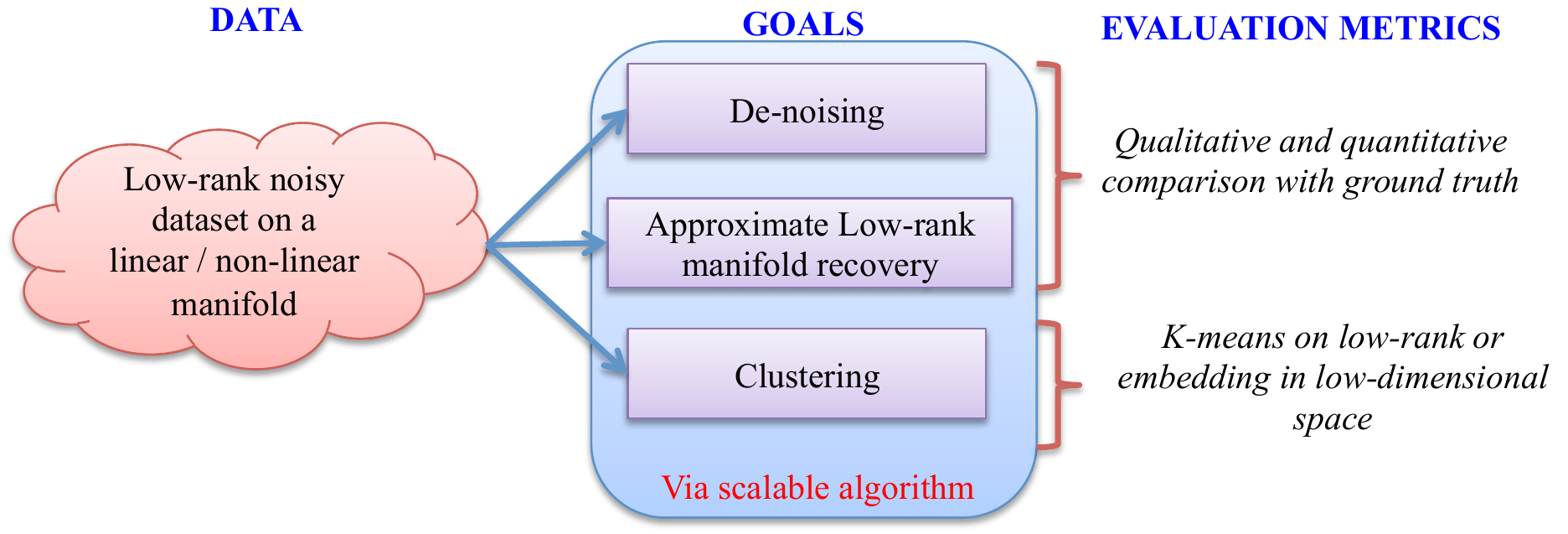}
         \caption{A summary of the main goals and the evaluation schemes used in this work.}
        \label{fig:goals}
    \end{figure*}
    
More specifically, the goal of this work is to solve the following problem: 

\begin{mdframed}[style=MyFrame]
\begin{center}
\textbf{OUR GOALS: \textit{\nauman{``Given a high dimensional big dataset that follows some unknown linear or non-linear structure in the low-dimensional space, recover the underlying low-rank structure in the data space via highly scalable and efficient algorithms''}}}. 
\end{center}
\end{mdframed}

 In the context of the above mentioned problem statement, we can simply define some basic goals of this work. Assume we have a high dimensional dataset, corrupted by noise and errors and consisting of millions of samples. Further assume that a clean representation of this data lies on a linear or non-linear low-dimensional subspace, then our goal is to perform the following while remaining in the original \textit{data space}:
\begin{enumerate}
\item De-noise the data.
\item Extract the low-dimensional manifold.
\item Be able to exploit the low-dimensional structure of the data
\item Handle corruptions and outliers
\item Do all above with highly scalable algorithms.
\end{enumerate}

It is slightly overwhelming that we want to target the above goals altogether.  We argue that the main aim of this work is not to provide exact solutions for the above tasks. In fact, our main motivation is just to provide a general scalable framework which provides approximate solutions for each of the above and works reasonably well in practice. We also provide theoretical bounds on the approximation errors for our framework. 

\section{A glimpse of our Proposed Framework \& Contributions}\label{sec:tools} 

Fig.~\ref{fig:venn} shows the general framework that we propose for GPCAG. GPCAG stands on three important pillars to extract low-rank manifold structures:
\begin{enumerate}
\item De-noising
\item Graph filtering
\item Data and graph compression
\end{enumerate}
It is quite obvious that we need a de-noising setup for real world data. In the following therefore we motivate and discuss the second tool in depth which was introduced in \cite{shahid2015fast}.  The compression framework is presented in \cite{shahid2016compressive} and will not be repeated here for brevity. 

\begin{figure*}[htbp]
    \centering
        \centering
        \includegraphics[width=1.0\textwidth]{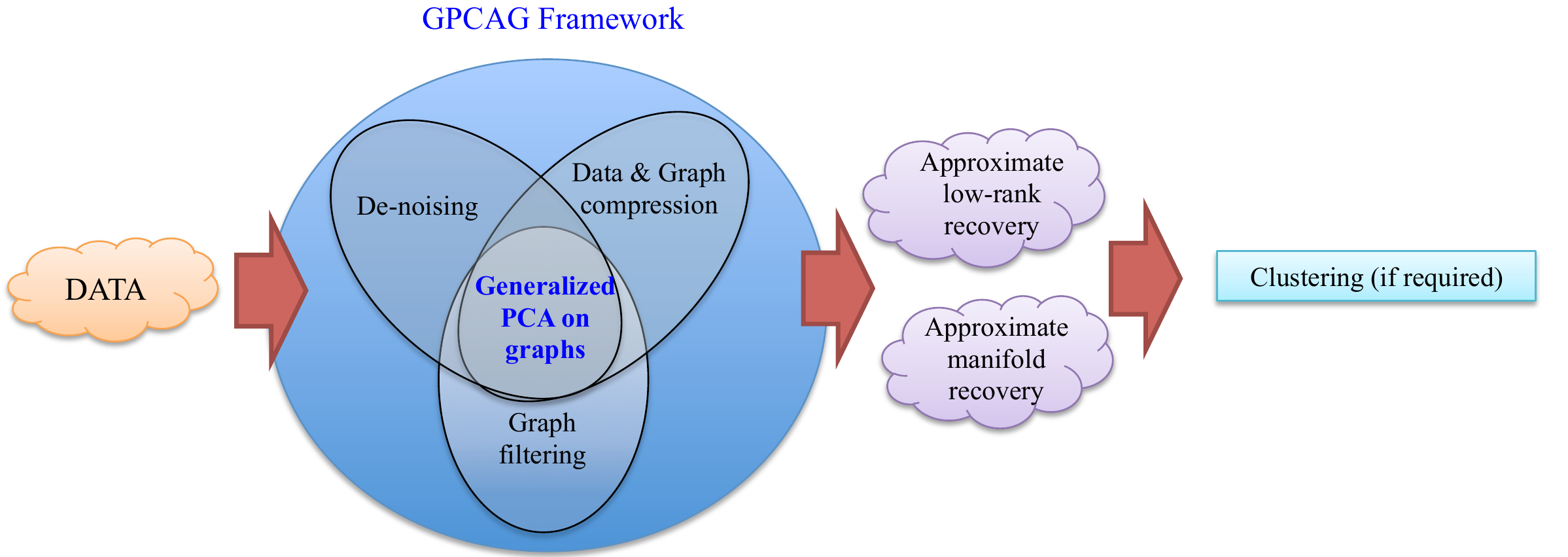}
         \caption{ The general framework that we propose for GPCAG: 1) De-noising, 2) Data and graph compression, 3) Graph filtering}
        \label{fig:venn}
    \end{figure*}
    
\begin{mdframed}[style=MyFrame]
\begin{center}
\textbf{ \textit{\nauman{``We propose  GRAPH as a building block for the low-rank extraction and recover low-rank ON a graph. This is different from Graph Regularized PCA. The graph is given, i.e, it can come from external source or can be constructed by using fast state-of-the-art methods. However, graph construction is not the focus of our work.''}}}. 
\end{center}
\end{mdframed}

\subsection{Low-rank Representation on Graphs: The key to our solution}
It is clear that the solution to our problem lies in some sort of dimensionality reduction framework which has a nice interpretation in the data space. As we target a generalized low-rank recovery for both linear and non-linear manifolds, we need tools from both linear and non-linear dimensionality reduction. 

\subsubsection{Robust Principal Component Analysis: The linear tool} It is not surprising that principal component analysis (PCA) \cite{abdi2010principal} is the most widely used tool for linear dimensionality reduction or low-rank recovery. For a dataset $Y\in \mathbb{R}^{p\times n}$ with $n$ $p$-dimensional data vectors, standard PCA learns the projections or principal components $\underline{V}\in \mathbb{R}^{n\times k}$ of $Y$ on a $k$-dimensional orthonormal basis $U\in \mathbb{R}^{p\times k}$, where $k < p$ (model 1 in Table~\ref{tab:models}). Though non-convex, this problem has a global minimum that can be computed using Singular Value Decomposition (SVD), giving a unique low-rank representation $X = U\underline{V}^{\top}$. In the following, without the loss of generality we also refer to $X = U\Sigma V^\top$ as the factorized form (SVD) of $X$, where  $\Sigma$ is a $k\times k$ diagonal matrix which consists of the singular values of $X$ and  $U^{\top}U = I_{k}$ and $V^\top V = I_{k}$. Note that $\underline{V}^{\top} = \Sigma V^\top $.

A main drawback of PCA is its sensitivity to heavy-tailed noise, i.e, a few strong outliers can result in erratic principal components. However, we want our low-rank extraction method to be robust to outliers as well. RPCA proposed by Candes et al. \cite{candes2011robust} overcomes this problem by recovering the clean low-rank representation $X$ from grossly corrupted $Y$ by solving model 4 in Table~\ref{tab:models}. Here $S$ represents the sparse matrix containing the errors and $\|X\|_{*}$ denotes the nuclear norm of $X$, the tightest convex relaxation of $\text{rank}(X)$ \cite{recht2010guaranteed}. Thus, RPCA is the tool that we will use from the class of linear dimensionality reduction algorithms.

\subsubsection{Graphs: The non-linear tool}
Graphs have the ability to model a wide variety of pairwise relationship between the data samples. Consider the real world examples of low-rank extraction which we described earlier. One can safely use graphs for each of those applications. For the surveillance application, the graph can encode the pairwise relationship between the different frames. One can either use a $\mathcal{K}$-nearest neighbors approach or just connect each frame to its two neighbors in time (a line graph). For the case of messages collected from a social network there can be several possibilities of graph construction. A straight forward way would be to construct a graph by using some additional information about the physical connectivity of the network, such as a graph of friends on the particular social platform. Another way would be to extract features from the messages and construct a graph using those features. For the case of a database of faces of several people the most meaningful graph would be between the pairwise faces. A good graph in such a case would be well connected between the faces of the same person. This will automatically give rise to a community structure in the dataset. For the case of time series data collected for a brain application one can construct a graph based on the anatomical connectivity of the brain regions or functional connectivity calculated directly from the brain signals.

Thus, the notion of the non-linearity in our framework follows directly from the use of graphs due to the above mentioned applications. In fact many non-linear dimensionality reduction techniques like Laplacian Eigenmaps (LE) \cite{belkin2003laplacian} use graphs to find the embedding of the data in a low-dimensional space. 

\subsection{The Proposed Framework}

\begin{mdframed}[style=MyFrame]
\begin{center}
\textbf{ \textit{\nauman{``Low-rank matrices on graphs (constructed between their rows and columns) can be recovered approximately from the uncompressed or compressed measurements via a dual graph regularization scheme.''}}}. 
\end{center}
\end{mdframed}

\begin{enumerate}

\item \textbf{Two graphs:} For a data matrix $Y \in \Re^{p \times n}$, we propose to construct two $\K$-nearest neighbor graphs: 1) between the rows $G_r$ and 2) between the columns $G_c$ of $Y$. The goal is to recover a low-rank representation $X$ of $Y$ on these graphs.
 \item \textbf{Low-rank matrices on graphs (LRMG):} In Section~\ref{sec:proposal} we present a first glimpse of our proposed framework and define a few basic concepts which provide the foundation for this work. More specifically, we introduce the novel concept of ``Low-rank matrices on graphs'' (LRMG).  We define a LRMG as a matrix whose rows and columns belong to the span of the first few eigenvectors of the graphs constructed between its rows ($G_r$) and columns ($G_c$).  This is a direct generalization of the concept of band-limited signals that is introduced in \cite{puy2015random}. Many real world datasets can benefit from this notion for low-rank feature extraction.
\item{ \textbf{Generalized Fast Robust PCA on Graphs (GFRPCAG):}  Sections \ref{sec:frpcag} and~\ref{sec:gffrpcag} provide a  step-by-step sound theoretical and algorithmic construction of a general method to recover LRMG. We call this method as ``Fast Robust PCA on Graphs'' as it resembles Robust PCA in its functionality but is actually a much faster method. We show that the recovery of LRMG depends on the spectral gaps of the Laplacians $\Larg_r$ and $\Larg_c$ of $G_r$ and $G_c$. More specifically we solve the following \textbf{dual graph filtering} problem:

\begin{equation}
\min_{X} \phi(Y-X) + \gamma_c \tr(X \Larg_c X^\top) + \gamma_r \tr(X^\top \Larg_r X),
\end{equation}
where $\phi$ is a convex loss function and $\gamma_r, \gamma_c$ are the model parameters. 

 Motivated from the deeper theoretical and experimental understanding of the approximation error of FRPCAG, we generalize the recovery to general graph filters. We design a family of deterministic filters which is characterized by a simple and easy to tune parameter and prove experimentally and theoretically that it improves upon the error incurred by FRPCAG.  We call this framework as ``Generalized FRPCAG (GFRPCAG)''. }

 \end{enumerate}
 
 Our proposed methods are  highly scalable and memory efficient. For a dataset $Y \in \Re^{p\times n}$, the complete generalized framework, involving graph construction, compression  and graph filtering and decoding poses a complexity that is $\mathcal{O}(n \log n)$ in the number of data samples $n$. In simple words, our framework can cluster 70,000 digits of MNIST  in less than a minute.  The experimental results for the complete framework are presented in \cite{shahid2015fast} and \cite{shahid2016compressive}.

\subsection{Organization of the work}
 A summary of all the main notations used in this work is presented in Table~\ref{tab:notations}. The rest of this work is organized as following:  In Section \ref{sec:graphs} we introduce graphs as the basic building blocks of our framework and present graph nomenclature. In Section~\ref{sec:prior} we discuss the state-of-the-art and connections and differences of our framework with the prior work. In Section \ref{sec:proposal} we present the concept of Low-rank matrices on graphs (LRMG) which is key to our proposed framework. In Section \ref{sec:frpcag}  we introduce Fast Robust PCA on Graphs (FRPCAG) as a recovery method for LRMG and present the optimization solution. In Section \ref{sec:theory} we present a sound theoretical analysis of FRPCAG and analyze the approximation error. In Section \ref{sec:examples} we present a few working of FRPCAG on the artificial and real world datasets to justify our theoretical findings. In Section \ref{sec:gffrpcag} we re-visit the apporximation error of FRPCAG and present a  generalized version GFRPCAG to recover LRMG with lower approximation error. Section \ref{sec:when} concludes the paper by discussing a few cases for which it makes sense to perform dual graph filtering. Please refer to \cite{shahid2015fast} for the experimental results.


 \begin{table*}[htbp]
\footnotesize
\caption{A summary of notations used in this work}
\centering
\resizebox{0.85\textwidth}{!}{\begin{tabular}[t]{| c | c | } \hline
\textbf{Notation}   & \textbf{Terminology} \\\hline
$\|\cdot \|_{F}$  & matrix frobenius norm \\\hline
$\|\cdot\|_{1}$ & matrix  $\ell_1$ norm \\\hline
$\| \cdot \|_{2,1}$ & matrix $\ell_{2,1}$ norm for column group sparsity \\\hline
$n$      & number of data samples \\\hline
$p$      & number of features / pixels \\\hline
$k$      & dimension of the subspace / rank of the data matrix \\\hline
${K}$  & number of classes in the data set \\\hline
$Y_0 \in \mathbb{R}^{p\times n} = U_0 \Sigma_0 V^{\top}_0$  & clean data matrix \& its SVD \\\hline
$Y \in \mathbb{R}^{p\times n} = U_y\Omega_y V^{\top}_y$  & noisy data matrix \& its SVD \\\hline
$X \in \mathbb{R}^{p\times n} = U\Sigma V^\top$  & low-rank noiseless approximation of $Y$ \& its SVD\\\hline
$\underline{V} (\underline{V}^\top = \Sigma V^\top) \in \mathbb{R}^{p\times k}$   &  principal components of $X$ \\\hline
$W \in \mathbb{R}^{n\times n}$ or $\mathbb{R}^{p\times p}$  & adjacency matrix between samples / features of $Y$ \\\hline
$D = diag(\sum_{j}W_{ij}) \forall i$      & diagonal degree matrix \\\hline
$\sigma$   & smoothing parameter / window width of the Gaussian kernel \\\hline
$G_{c}$  & graph between the samples / columns of $Y$ \\\hline
$G_r$  & graph between the features / rows of $Y$ \\\hline
$\mathcal{(V,E)}$ & set of vertices, edges for graph \\\hline
$\gamma_r$  & penalty for $G_r$ Tikhonov regularization term \\\hline
$\gamma_c$  & penalty for $G_c$ Tikhonov regularization term \\\hline
$\mathcal{K}$   & number of nearest neighbors for the construction of graphs \\\hline
$C_c $  & sample covariance of $Y$ \\\hline
$C_r$   & feature covariance of $Y$ \\\hline
$\Larg_r \in \mathbb{R}^{n\times n} = P\Lambda_r P^\top = P_{k_r}\Lambda_{k_r}P^{\top}_{k_r} + \bar{P}_{k_r}\bar{\Lambda}_{k_r}\bar{P}^{\top}_{k_r}$ & Laplacian for graph $G_r$ \& its EVD \\\hline
$\Larg_c \in \mathbb{R}^{p\times p} = Q\Lambda_c Q^\top = Q_{k_c}\Lambda_{k_c}Q^{\top}_{k_c} + \bar{Q}_{k_c}\bar{\Lambda}_{k_c}\bar{Q}^{\top}_{k_c}$ &  Laplacian for graph $G_c$ \& its EVD \\\hline
$P_{k_r} \in \Re^{p \times k_r}, Q_{k_c} \in \Re^{n \times k_c}$ & First $k_r, k_c$ eigenvectors in $P,Q$  \\\hline
$\bar{P}_{k_r} \in \Re^{p \times (p - k_r)}, \bar{Q}_{k_c} \in \Re^{n \times (n - k_c)}$ &  eigenvectors beyond $k_r + 1, k_c + 1$ in $P,Q$   \\\hline
$\Lambda_{k_r} \in \Re^{k_r \times k_r}, \Lambda_{k_c} \in \Re^{k_c \times k_c}$ & $\Lambda_r, \Lambda_c$ with the first $k_r, k_c$ eigenvalues  \\\hline
$\bar{\Lambda}_{k_r} \in \Re^{(p - k_r) \times (p - k_r)}, \bar{\Lambda}_{k_c} \in \Re^{(n - k_c) \times (n - k_c)}$ & $\Lambda_r, \Lambda_c$ with all the eigenvalues above $k_r, k_c$ \\\hline
\end{tabular}}
\label{tab:notations}
\end{table*}

\section{Graphs: The basic building block!}\label{sec:graphs}
\subsection{Graph Nomenclature}
A graph is a tupple ${G}=\{ \V,\E,\mathcal{W}\}$ where $\V$ is a set of vertices, $\E$ a set of edges,
and $\W : \V \times \V \rightarrow \Rbb_+$ a weight function. The vertices are indexed from $1,\dots, |\V|$ and each entry  of the weight matrix $W \in \mathbb R^{|\V|\times |\V|}_+$ contains the weight of the edge connecting the corresponding vertices: $W_{i,j} = \W (v_i,v_j)$. If there is no edge between two vertices, the weight is set to $0$. We assume $W$ is symmetric, non-negative and with zero diagonal. We denote by $i\leftrightarrow j$ the connection of node $v_i$ to node $v_j$. 

\subsubsection{Standard graph construction methods}
There can be several ways to define a graph ${G}$ for a dataset. For example, for a dataset $Y \in \Re^{p\times n}$, the vertices $v_i$  of the graph ${G}$ can correspond to the data samples $y_i \in \Re^{p}$. The most common method to construct graphs is via a standard $\K$-nearest neighbors strategy. The first step consists of searching the closest neighbors for all the samples using Euclidean distances.  Thus, each $y_i$ is connected to its $\mathcal{K}$ nearest neighbors $y_j$, resulting in $|\mathcal{E}|\approx K n $ number of connections. The second step consists of computing the graph weight matrix $W$ using one of the several commonly used strategies \cite{jin2014low}. We mention some of the most commonly used schemes below:

\textbf{Gaussian kernel weighting}
\begin{equation*}
W_{ij} = \begin{cases}
\exp\Big(-\frac{\|(y_i-y_j)\|^{2}_{2}}{ \sigma^{2}}\Big) & \text{if $y_j$ is connected to $y_i$}\\
0 & \text{otherwise.}\\
\end{cases}
\end{equation*}

\textbf{Binary weighting}
\begin{equation*}
W_{ij} = \begin{cases}
1 & \text{if $y_j$ is connected to $y_i$}\\
0 & \text{otherwise.}\\
\end{cases}
\end{equation*}

\textbf{Correlation weighting}
\begin{equation*}
W_{ij} = \begin{cases}
\frac{y^{\top}_i y_j}{\|y_i\|_2 \|y_j\|_2 } & \text{if $y_j$ is connected to $y_i$}\\
0 & \text{otherwise.}\\
\end{cases}
\end{equation*}

One can also use $\ell_1$ distance weighting scheme if the data is corrupted by outliers. 

For a vertex $v_i\in \V$ (or a data sample $y_i \in Y$), the degree $d(i)$ is defined as the sum of the weights of incident edges: $d(i)=\sum_{j \leftrightarrow i } W_{i,j}$. Finally, the graph can be characterized using a normalized or unnormalized graph Laplacian. The normalized graph Laplacian $\Larg_n$ defined as
$$\Larg_n = D^{-\frac{1}{2}} (D-W) D^{-\frac{1}{2}}= I - D^{-\frac{1}{2}} W D^{-\frac{1}{2}}$$ where $D$ is the diagonal degree matrix with diagonal entries $D_{ii}=d(i)$ and $I$ the identity and the un-normalized is defined as $\Larg = D-W.$ 

\subsubsection{Fast Methods for Graph Construction}
For big or high dimensional datasets, i.e, large $n$ or large $p$ or both, the exact $\mathcal{K}$-nearest neighbors strategy can become computationally cumbersome $\mathcal{O}(n^2)$. In such a scenario the computations can be made  efficient using the FLANN library (Fast Library for Approximate Nearest Neighbors searches in high dimensional spaces) \cite{muja2014scalable}. The quality of the graphs constructed using this strategy is slightly lower as compared to the above mentioned strategy due to the approximate nearest neighbor search method. However,  for a graph of $n$ nodes the computationally complexity is as low as $\mathcal{O}(n\log n)$ \cite{sankaranarayanan2007fast}. Throughout this work we construct our graphs using FLANN irrespective of the size of the datasets.

\subsubsection{Graph Signal}
In this framework, a graph signal is defined as a function $s: \V \rightarrow  \Rbb$ assigning a value to each vertex. It is convenient to consider a signal $s$ as a vector of size $|\V|$ with the $i^{\mathrm{th}}$ component representing the signal value at the $i^{\mathrm{th}}$ vertex. 
For a signal $s$ living on the graph $G$, the gradient $\nabla_{G} :  \Rbb^{|\V|} \rightarrow \Rbb^{|\E|}$  is defined as
$$
\nabla_{G}s (i,j) = \sqrt{W(i,j)}\left(\frac{s(j)}{\sqrt{d(j)}}-\frac{s(i)}{\sqrt{d(i)}}\right),
$$
where we consider only the pair $\{i,j\}$ when $i \leftrightarrow j$. For a signal $c$ living on the graph edges, the adjoint of the gradient $\nabla_{G}^*: \Rbb^{ |\E|} \rightarrow \Rbb^{|\V|} $, called divergence can be written as
\begin{eqnarray*}
\nabla_{G}^*c(i) &= & \sum_{i\leftrightarrow j} \sqrt{W(i,j)}\left(\frac{1}{\sqrt{d(i)}}c(i,i) -\frac{1}{\sqrt{d(j)}}c(i,j)\right).
\end{eqnarray*}

The unnormalized Laplacian corresponds to the second order derivative and its definition arises from $\Larg s :=  \nabla_{\G}^* \nabla_{\G} s$.

\subsubsection{Spectral Graph Theory: The Graph Fourier Transform}
Since the Laplacian $\Larg$ is by construction always a symmetric positive semi-definite operator, it always possesses a complete set of orthonormal eigenvectors that we denote by $\{ q_\ell \}_{\ell=0,1,..., n-1}$. For convenience, and without loss of generality, we order the set of eigenvalues as follows: $0=\lambda_0 < \lambda_1 \leq \lambda_2 \leq ... \leq \lambda_{n-1} = \lambda_{\rm max}$.
When the graph is connected, there is only one zero eigenvalue. In fact, the multiplicity of the zero eigenvalue is equal to the number of connected components. See for example \cite{chung1997spectral} for more details on spectral graph theory. For simplicity, we use the following notation: $$\Larg = Q\Lambda Q^\top. $$
For symmetric (Hermitian) matrices, we define the following operator:
$$
g(\Larg) = Q g(\Lambda) Q^\top,
$$
where $g(\Lambda)$ is a diagonal matrix with  $g(\Lambda)_{\ell,\ell} = g(\lambda_\ell)$. 

One of the key idea of this spectral decomposition is that the Laplacian eigenfunctions are used as a graph "Fourier" basis \cite{shuman2013emerging}.
The graph Fourier transform $\hat{x}$ of a signal $x$ defined on a graph $G$ is the projection onto the orthonormal set of eigenvectors $Q$ of the graph Laplacian associated with $G$:
\begin{equation} \label{def: graph Fourier transform}
\hat{x} = Q^\top x
\end{equation}
Since $Q$ is an orthonormal matrix, the inverse graph Fourier transform becomes:
\begin{equation}
x = Q \hat{x}
\end{equation}
The spectrum of the Laplacian replaces the frequencies as coordinates in the Fourier domain with the following relation $|f_\ell|^2=\lambda_\ell$. 
The main motivation of this definition is the fact that
the Fourier modes are intuitively the oscillating modes for the frequencies given by the eigenvalues \cite{shuman2013emerging} \cite{shuman2016vertex}.  A very important property of these modes is  that the higher the eigenvalue, the more the mode oscillates. As a consequence, high eigenvalues are associated to high frequencies and low eigenvalues correspond to low frequencies.

\subsubsection{Smooth Signals on Graphs } \label{sec:meaning_of_xLx}
In many problems, $tr(X^\top \Larg X)$ is used as a smoothness regularization term for the signals contained in $X$. The most common interpretation is that it is equivalent to penalizing the gradient of the function, i.e:
\begin{equation}
tr(X^\top \Larg X) = \| \nabla_{G} X  \|_F^2.
\end{equation}
However, we can also study this term from a different angle. Using the graph Fourier transform, it is possible to rewrite it as a spectral penalization:
\begin{equation}
tr(X^\top \Larg X) = tr(X^\top Q \Lambda Q^\top X) =  tr(\hat{X}^\top \Lambda \hat{X}) = \| \Lambda^{\frac{1}{2}} \hat{X} \|_F^2.
\end{equation}
From the previous equation, we observe that the regularization term $tr(X^\top \Larg X) $ penalizes the Fourier transform of the signal with the weights $\sqrt{\Lambda}$. This point of view suggests a new potential regularization term. We can use a function $g:\Rbb_+ \rightarrow \Rbb_+$, to change the Frequency weighting:
\begin{equation}
tr(X^\top g(\Larg) X) = tr(\hat{X}^\top g(\Lambda) \hat{X}) = \| g(\Lambda)^{\frac{1}{2}} \hat{X} \|_F^2.
\end{equation}
This technique will be a good asset to improve the quality of the low-rank representation, as described in the following sections.

\section{Prior Work: Connections \& Differences with Our Framework}\label{sec:prior}

\begin{mdframed}[style=MyFrame]
\begin{center}
\textbf{ \textit{\nauman{``The goal is NOT to improve  low-rank representation via graph regularization. Instead, we aim to solely recover an approximate low-rank matrix with dual-graph regularization only. We show that one can obtain a good enough low-rank representation without using expensive nuclear norm or non-convex matrix factorization.''}}}. 
\end{center}
\end{mdframed}

  \begin{table*}[htbp]
\footnotesize
\caption{A comparison of various PCA models and their properties. $Y \in \mathbb{R}^{p\times n}$ is the data matrix, $U \in \mathbb{R}^{p\times k}$ and $\underline{V} \in \mathbb{R}^{n\times k}$ are the {principal directions} and  {principal components} in a $k$ dimensional linear space (rank = $k$). $X = U\underline{V}^\top \in \mathbb{R}^{p\times n}$ is the  {low-rank representation} and $S \in \mathbb{R}^{p\times n}$ is the sparse matrix. $\Larg$, $\Larg^{g}$ and $\Larg_{h}^{g} \in \mathbb{S}^{n\times n}$ characterize a simple graph or a hypergraph $G$ between the samples of $X$. $\|\cdot\|_{F}$, $\|\cdot\|_{*}$ and $\|\cdot\|_{1}$ denote the Frobenius, nuclear and $l_{1}$ norms respectively.}
\centering
\resizebox{1.0\textwidth}{!}{\begin{tabular}[t]{| c | c | c | c | c | c | c | c |} \hline
   &   \textbf{Model}   & \textbf{Objective}  &   \textbf{Constraints}   & \textbf{Parameters}   &  \textbf{Graph?} & \textbf{Factors?}  & \textbf{Convex?}   \\\hline
1 &  PCA    &  $ \min_{U,Q} \|Y-U\underline{V}^\top\|_{F}^{2}$  &   $U^{T}U  = I$ & $k$ & no  &    yes &  no   \\\hline
2   & RPCA \cite{candes2011robust}  &  $\min_{X,S} \|X\|_{*} + \lambda\|S\|_{1}$  &   $Y = X + S$  & $\lambda$  & no  &  no & yes  \\\hline
3   & RPCAG \cite{shahid2015robust} & $ {\min_{X,S} \|X\|_{*} + \lambda\|S\|_{1} + \gamma \tr (X\Larg X^{T})} $  & $Y= X + S$ & $\lambda, \gamma$    &  yes &  no   & yes \\\hline
4  &  GLPCA  \cite{jiang2013graph}   & $\min_{U,\underline{V}} \|X-U\underline{V}^\top\|_{F}^{2}  + \gamma \tr(\underline{V}^\top \Larg \underline{V})$ & $\underline{V}\underline{V}^\top = I $  &     &  &     & \\\cline{1-3}
5  & RGLPCA  \cite{jiang2013graph} & $\min_{U,\underline{V}} \|X-U\underline{V}^\top\|_{2,1}  +  \gamma \tr(\underline{V}^\top \Larg \underline{V}) $ &  &  $k,\gamma$  & &   &  \\\cline{1-4}
6  & MMF  \cite{zhang2013low}      &  $\min_{U,\underline{V}} \|X-U\underline{V}^\top\|_{F}^{2}  +  \gamma \tr(\underline{V}^\top \Larg \underline{V})$ & $U^{T}U = I $   &   & yes & yes & no \\\cline{1-5}
7  & MMMF   \cite{tao2014low} & $\min_{U,\underline{V},\boldsymbol{\alpha}} \|X-U\underline{V}^\top\|_{F}^{2}  + \gamma\tr(\underline{V}^\top(\sum_{g}{\alpha}_{g}\Phi^{g}) \underline{V}) + \beta\|\boldsymbol{\alpha}\|^{2}$ & $U^{T}U = I $ & $k,\gamma,\beta$ &  &   &  \\\cline{1-3}
8  & MHMF  \cite{jin2014low}  & $\min_{U,\underline{V},\boldsymbol{\alpha}} \|X-U\underline{V}^\top\|_{F}^{2}  + \gamma\tr(\underline{V}^\top(\sum_{g}{\alpha}_{g}\Phi_{h}^{g}) \underline{V}) + \beta\|\boldsymbol{\alpha}\|^{2}$ &  $\boldsymbol{1}^{T}\boldsymbol{\alpha} = \boldsymbol{1}$  &   &  &  & \\\hline
9 & LRR  \cite{liu2013robust}  & $\min_{X} \|X\|_{*} + \lambda \|S\|_1$  & $Y = YX + S$ &  $\lambda$ & no & no & yes \\\cline{1-3}
10 & GLRR  \cite{lu2013graph}  & $\min_{X} \|X\|_{*} + \lambda \|S\|_1 + \gamma \tr(X\Larg X^\top)$  &  &  &  & &  \\\hline
\end{tabular}}
\label{tab:models}
\end{table*}

\subsection{Graph regularized PCA}
The idea of extracting enhanced low-rank and sparse representations using graphs has been around for a while now. Many recent works  propose to incorporate the data manifold information in the form of a discrete graph into the dimensionality reduction framework \cite{jiang2013graph, zhang2013low, gao2013laplacian, cai2011graph, tao2014low,jin2014multiple,jin2014low,peng2015enhanced,du2015sparse}.  In fact, for PCA, this can be considered as a method of exploiting the local smoothness information in order to improve low-rank and clustering quality. 

Depending on how the graph information is used in the context of PCA we divide the works into two types:
\begin{itemize}
\item The graph encodes the smoothness of principal components $\underline{V}$. Such methods explicitly learn the basis $U$ and components $\underline{V}$ and we refer to such methods as \textit{factorized models}.
\item The graph encodes the smoothness of the low-rank matrix $X$ as a whole and we refer to such methods as \textit{non-factorized models}
\end{itemize}

\subsubsection{Factorized Models}
The graph smoothness of the principal components $\underline{V}$ using the graph Laplacian $\Larg$ has been exploited in various works that explicitly learn $\underline{V}$ and the basis $U$. We refer to such models as the ones which use \textit{principal components graph}. In this context Graph Laplacian PCA (GLPCA) was proposed in \cite{jiang2013graph} (model 4 in Table~\ref{tab:models}). This model explicitly learns the basis $U$ and principal components $\underline{V}$ by assuming that $\underline{V}$ is smooth on the graph of data samples. Note that as compared to the standard PCA, the model requires $\underline{V}$ to be orthonormal, instead of $U$. The closed form solution to this problem is given by the eigenvalue decomposition operation. The model is non-convex, however a globally unique solution can be obtained for small datasets by the closed form solution.

Manifold Regularized Matrix Factorization (MMF) \cite{zhang2013low} (model 6 in Table~\ref{tab:models}) is another such factorized model which explicitly learns the basis $U$ and principal components $\underline{V}$. Note that like standard PCA, the orthonormality constraint in this model is on $U$, instead of the principal components $\underline{V}$. This model implicitly encodes the smoothness of the low-rank matrix $X$ on the graph as  $\tr(X\Larg X^\top) = \tr(\underline{V}^\top \Larg \underline{V})$, using $X = U\underline{V}^\top$.

More recently, the idea of using multiple graphs has been seen in the PCA community. In this context the smoothness of the principal components $\underline{V}$ has been proposed on the linear combination of the graphs. Such models are particularly useful when the data is noisy and one cannot rely on a single graph construction strategy. A linear combination of graphs constructed using various strategies might provide robustness to noise. Tao et al. \cite{tao2014low} propose an ensemble of graphs in this context, whereas Jin et al. \cite{jin2014low} propose to use an ensemble of hypergraphs respectively ($7^{th}$ and $8^{th}$ models in Table~\ref{tab:models}). All of the above mentioned models are non-convex and require a rank $k$ to be specified as a model parameter.

Another important line of works related to the factorized models  computes a non-negative low-rank representation (NMF) for data by decomposing a data matrix as a product of two thin non-negative matrices \cite{lee1999learning}. Several models which incorporate graph regularization to NMF have been proposed in the literature, like \cite{cai2011graph}. 

\subsubsection{Non-factorized Models}
In contrast to the factorized models, the non-factorized models directly learn a low-rank matrix $X$ by decomposing the data matrix $Y$ into sum of a low-rank $X$ and sparse matrix $S$. Such models are convex.
The authors of \cite{shahid2015robust} have generalized robust PCA \cite{candes2011robust} by incorporating the graph smoothness (model 2 in Table~\ref{tab:models}). In their model, they propose a simple graph smoothness regularization term directly on the low-rank matrix instead of principal components and improve the clustering and low-rank recovery properties. They call it Robust PCA on Graphs (RPCAG). We refer to such models as the ones which use \textit{low-rank graph}.

Another line of works, known as Low-rank representations (LRR) \cite{liu2013robust} (model 9 in Table~\ref{tab:models}) falls under the category of convex non-factorized models. More specifically, LRR corresponds to a generalization of RPCA to multiple subspaces. The framework is a dictionary learning problem, where one would like to encode each data sample as a low-rank combination of the atoms of the dictionary. In its most basic form, LRR uses the data matrix itself as a dictionary. Later on the authors of \cite{lu2013graph},\cite{liu2014enhancing},\cite{wang2015low} proposed to incorporate the smoothness of the low-rank codes on a graph into the LRR framework (model 10 in Table~\ref{tab:models}). This enhances the subspace recovery capability of the simple LRR. An extension of the graph regularized LRR with local constraints for graph construction has also been proposed in \cite{zheng2013low}.

\subsection{Shortcomings of the state-of-the-art}
Both the factorized and the non-factorized PCA based models suffer from a few problems. The factorized models require the rank $k$ to be specified as a model parameter. One never knows the exact rank of the real world data, therefore several parameters need to be tried to get an optimal solution. The specification of the rank makes these algorithms scalable to some extent because they do not need a full SVD. However, the non-convexity requires the algorithms to be run several times for each parameter model unless we know a good initialization scheme. Furthermore, GLPCA and MMF are not robust to gross errors and outliers in the dataset. In fact making them robust to outliers requires a modification of objective function, thus adding another model parameter.

The non-factorized models on the other hand are convex but they are not scalable. RPCA and  RPCAG require the full SVD of the estimated low-rank matrix in every iteration of the algorithm. For a data matrix $Y \in \Re^{p\times n}$ the complexity of SVD is $\mathcal{O}(np^2)$ where $p < n$. LRR and GLRR are even more computationally complex because they require the SVD of $n\times n$ matrix which costs $\mathcal{O}(n^3)$. As already demonstrated with simple examples of Fig.~\ref{fig:demo}, RPCA recovers only linear structures in the data. Although RPCAG, incorporates graph structure in the standard RPCA framework, the issue of scalability remains there. The randomized algorithms reduce the computational complexity of these algorithms but they still require SVD on the compressed data which is still an issue for very big datasets.

\subsection{Connections and differences with the state-of-the-art}
The idea of using two graph regularization terms has previously appeared in the work of matrix completion \cite{kalofolias2014matrix}, co-clustering \cite{gu2009co}, NMF \cite{shang2012graph}, \cite{BenziKBV16arxiv} and more recently in the context of low-rank representation \cite{yin2015dual}. However, to the best of our knowledge all these models aim to improve the clustering quality of the data in the low-dimensional space.  The co-clustering \& NMF based models which use such a scheme \cite{gu2009co}, \cite{shang2012graph} suffer from non-convexity and the works of  \cite{kalofolias2014matrix} and \cite{yin2015dual} use a nuclear-norm formulation which is computationally expensive and not scalable for big datasets. Our proposed method is different from these models in the following sense:

\begin{itemize}
\item We do not target an improvement in the low-rank representation via graphs. Our method aims to solely recover an approximate low-rank matrix with dual-graph regularization only. The underlying motivation is that one can obtain a good enough low-rank representation without using expensive nuclear norm or non-convex matrix factorization. Note that the NMF-based method   \cite{shang2012graph} targets the smoothness of factors of the low-rank while the co-clustering \cite{gu2009co} focuses on the smoothness of the labels. \textit{Our method, on the other hand, targets directly the recovery of the low-rank matrix, and not the one of the factors or labels}.
\item  We introduce the concept of low-rank matrices on graphs and provide a theoretical justification for the success of our model. The  use of PCA as a scalable and efficient clustering method using dual graph regularization has surfaced for the very first time in this paper.  
\end{itemize}

\section{Towards Low-Rank Matrices on Graphs}\label{sec:proposal}

\begin{mdframed}[style=MyFrame]
\begin{center}
\textbf{A NOVEL CONCEPT:  \textit{\nauman{``A matrix is said to be low-rank on graphs (LRMG) constructed between its rows $G_r$ and columns $G_c$ if its rows belong to the span of the eigenvectors of $G_c$ and columns to the span of the eigenvectors of $G_r$.''}}}. 
\end{center}
\end{mdframed}

In this section we define the novel concept of low-rank matrices on graphs which motivates a fast recovery method to approximate Robust PCA.
\subsection{A General Insight: We need two graphs!}\label{sec:motivation}
Consider  a simple example  of the digit $3$ from the traditional USPS dataset (example taken from \cite{2016arXiv160102522P}). We vectorize all the images and form a data matrix $Y$, whose columns consist of different samples of digit $3$ from the USPS dataset.  We build the $10$ nearest neighbors graph of features $G_r$,  i.e, a graph between the rows of $Y$ and $G_c$, i.e, a graph between the columns of $Y$ using the FLANN strategy of Section \ref{sec:graphs}. Let $\Larg_r \in \mathbb{R}^{p \times p} = D_r - W_r = P \Lambda_r P^\top $ and $\Larg_c \in \mathbb{R}^{n \times n} = D_{c} - W_c = Q\Lambda_c Q^\top$ as the normalized graph Laplacians for the graphs $G_r, G_c$ between the rows and columns of the data $Y$.

\subsubsection{Graph of features provides a basis}
   Fig.~\ref{fig:exp3} shows the eigenvectors of the normalized Laplacian $\Larg_r$ denoted by $P$. We observe that they have a $3$-like shape. In Fig.~\ref{fig:exp3}, we also plot the eigenvectors associated to the experimental covariance matrix
$C_r$\footnote{The experimental covariance matrix is computed as $C_r = \frac{\tilde{Y}\tilde{Y}^\top}{n}$, where $n$ is the number of samples and $\tilde{Y} = Y - \mu_Y$ for $\mu_Y = \frac{1}{n}\sum_{j=1}^n Y_{ij}.$ This definition is motivated in \cite{2016arXiv160102522P}.}. We observe that both sets of eigenvectors are similar. This is confirmed by computing the following matrix:
\begin{equation} \label{eq:Fourier_cov_matrix}
\Gamma_r = P^\top C_r P
\end{equation}
In order to measure the level of alignment between the orthogonal basis $P_r$ and the basis of $C$, we use the following ratio, which we call as the \textbf{\textit{alignment order}}:
\begin{equation}
s_r(\Gamma_r) = \left(\frac{\sum_\ell \Gamma_{r \ell,\ell}^2}{\sum_{\ell_1}\sum_{\ell_2} \Gamma_{r \ell_1,\ell_2}^2} \right)^{\frac{1}{2}} =\frac{\| \rm{diag}(\Gamma_r) \|_2}{\| \Gamma_r \|_F}.
\end{equation}
When the two bases are aligned, the covariance matrix $C_r$ and the graph Laplacian $\Larg_r$ are simultaneously diagonalizable, giving a ratio equal to $1$. On the contrary, when the bases are not aligned, the ratio is close to $\frac{1}{p}$, where $p$ is the dimension of the dataset. Note that an alternative would be to compute directly the inner product between $P$ and the eigenvectors of $C_r$. However, using $\Gamma_r$ we implicitly weight the eigenvectors of $C_r$ according to their importance given by their corresponding eigenvalues. 

In the special case of the digit $3$, we obtain a ratio $s_r(\Gamma_r) = 0.97$, meaning that the main covariance eigenvectors are well aligned to the graph eigenvectors. Fig.~\ref{fig:exp3} shows a few eigenvectors of both sets and the matrix $\Gamma_r$. 
\begin{figure*}[htb!]
\begin{center}
\includegraphics[width=0.3\textwidth]{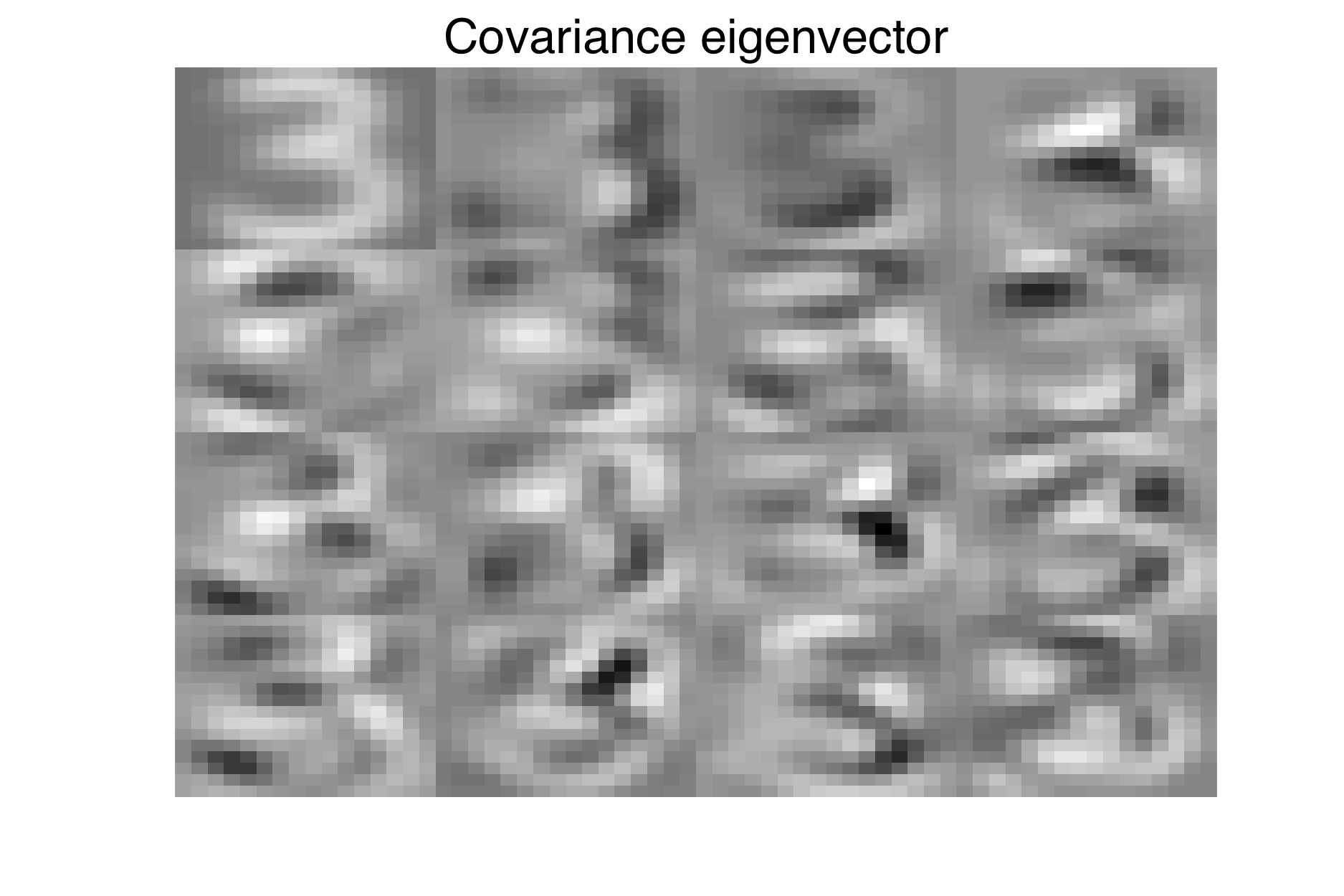} 
\includegraphics[width=0.3\textwidth]{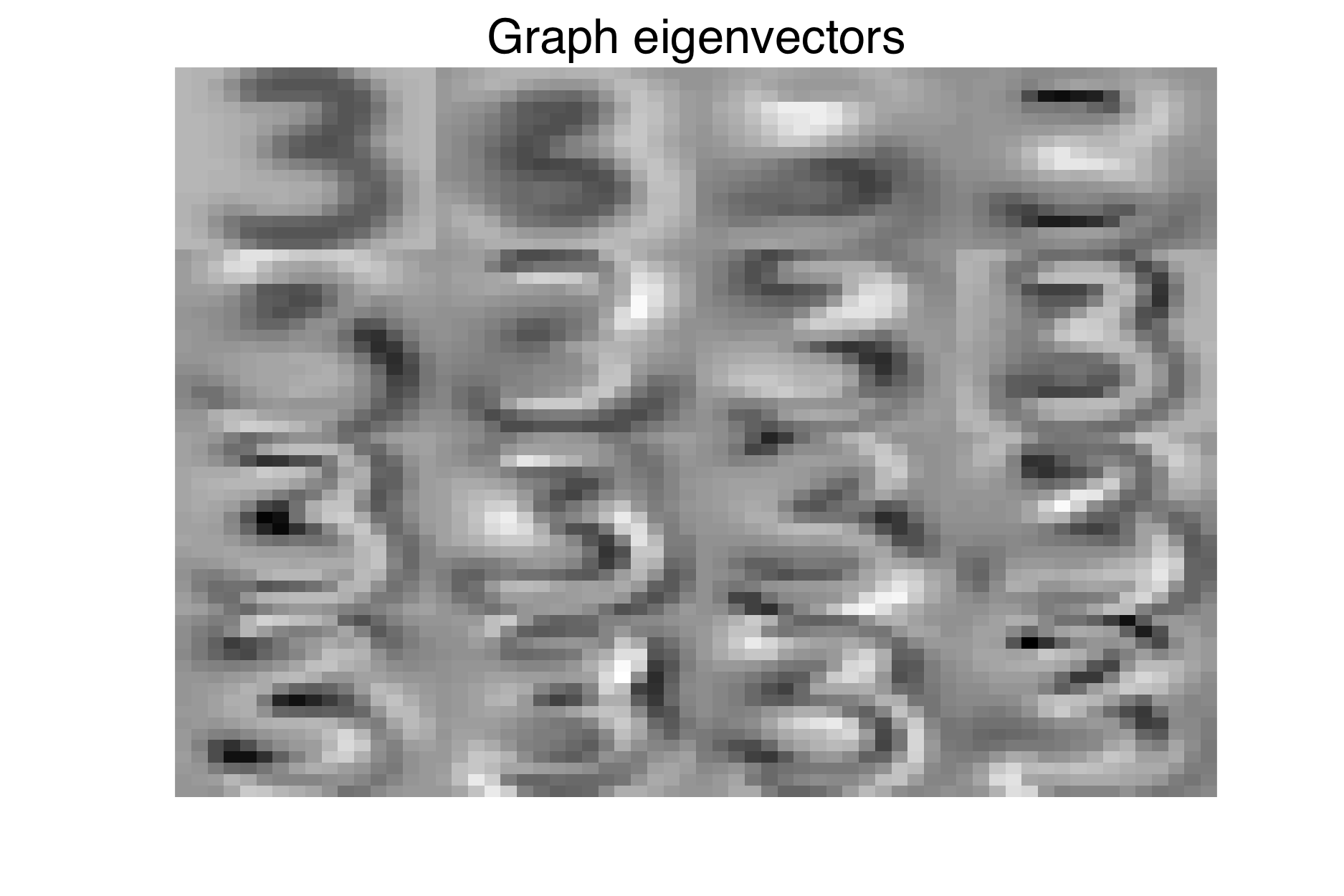} 
\includegraphics[width=0.3\textwidth]{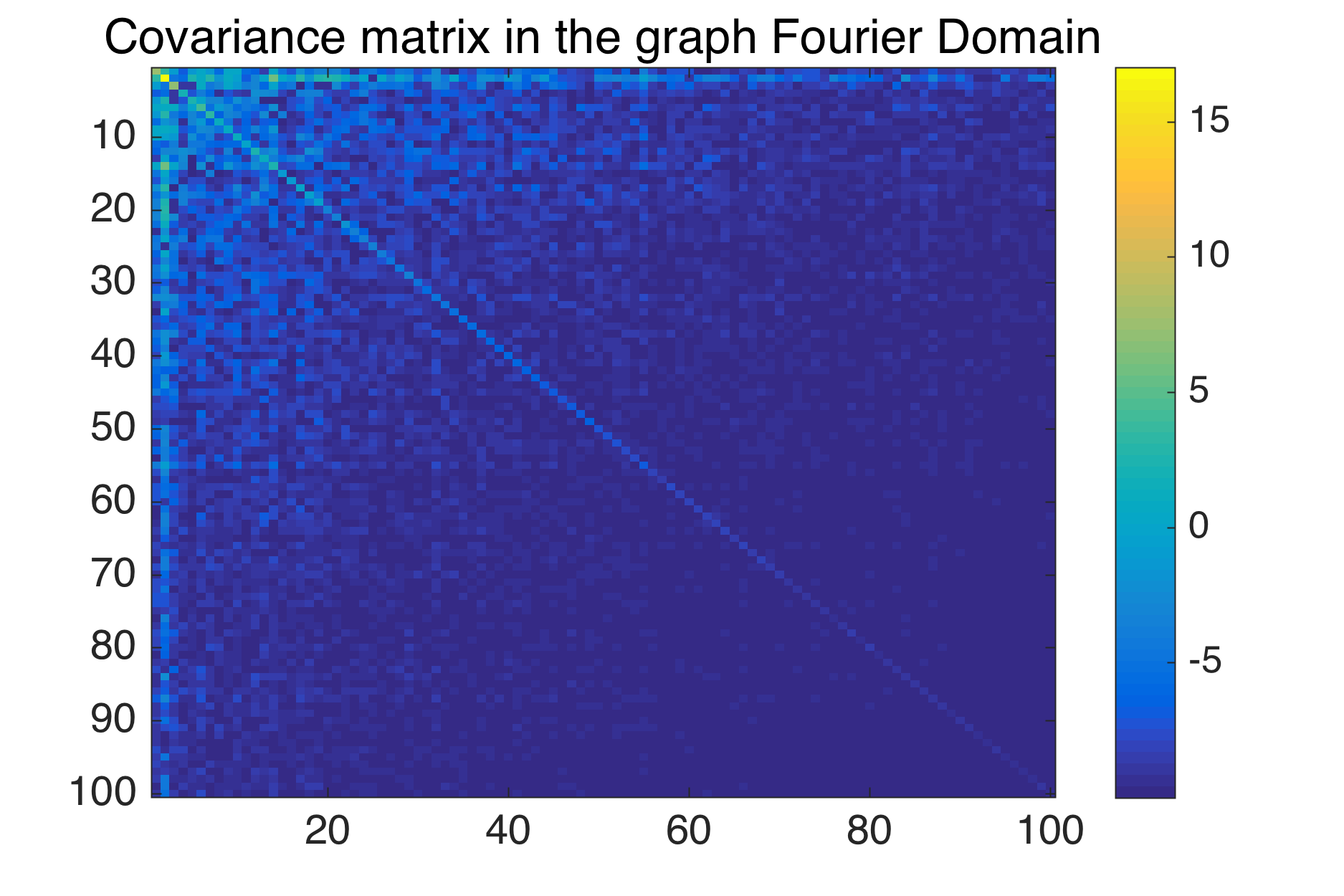}
\end{center}
\caption{Studying the number $3$ of USPS. Left: Covariance eigenvectors associated with the $16$ highest eigenvalues. Right: Laplacian eigenvectors associated to the $16$ smallest non-zero eigenvalues. Because of stationarity, Laplacian eigenvectors are similar to the covariance eigenvectors. Right: $\Gamma_r = P^\top C_r P$ in dB. Note the diagonal shape of the matrix implying that $P$ is aligned with the eigenvectors of $C_r$. (Figure taken from \cite{2016arXiv160102522P} with permission).}
\label{fig:exp3}
\end{figure*}
This effect has been studied in \cite{2016arXiv160102522P} where the definition of stationary signals on graphs is proposed. A similar idea is also the motivation of the Laplacianfaces algorithm~\cite{he2005face}.

 A closer look at the rightmost figure of Fig.~\ref{fig:exp3} shows that most of the energy in the diagonal is concentrated in the first few entries. To study how well is the alignment of the first few eigenvectors in P and $C_r$ we compute another ratio, which we call as the \textbf{\textit{rank-k alignment order}}:

\begin{equation}\label{eq:rankkalign}
\hat{s}_r(\Gamma_r) = \frac{\sum_{\ell = 1}^{k} \Gamma_{r \ell,\ell}^2}{\sum_{\ell} \Gamma_{r \ell,\ell}^2} 
\end{equation}

This ratio denotes the fraction of energy which is concentrated in the first $k$ entries of the diagonal.  For the example of the USPS dataset above, $\hat{s}_r(\Gamma_r) = 0.99$ for $k = 10$. This shows that $99\%$ of the diagonal energy is concentrated in the first ten entries.  Thus first 10 eigenvectors of the Laplacian are very well aligned with the first 10 eigenvectors of the covariance matrix. This phenomena implies that the digit $3$ of the USPS dataset is low-rank, i.e, only the first few eigenvectors (corresponding to the low eigenvalues) are enough to serve as the features for this dataset.

Of course, this phenomena also holds for the full dataset. Let us analyze how the graph eigenvectors evolve when all digits are taken into account. Fig.~\ref{fig:exp4} shows the Laplacian and covariance eigenvectors for the full USPS dataset. Again we observe some alignment: $s_r(\Gamma_r) = 0.82$. 
\begin{figure}[htb!]
\begin{center}
\includegraphics[width=0.3\textwidth]{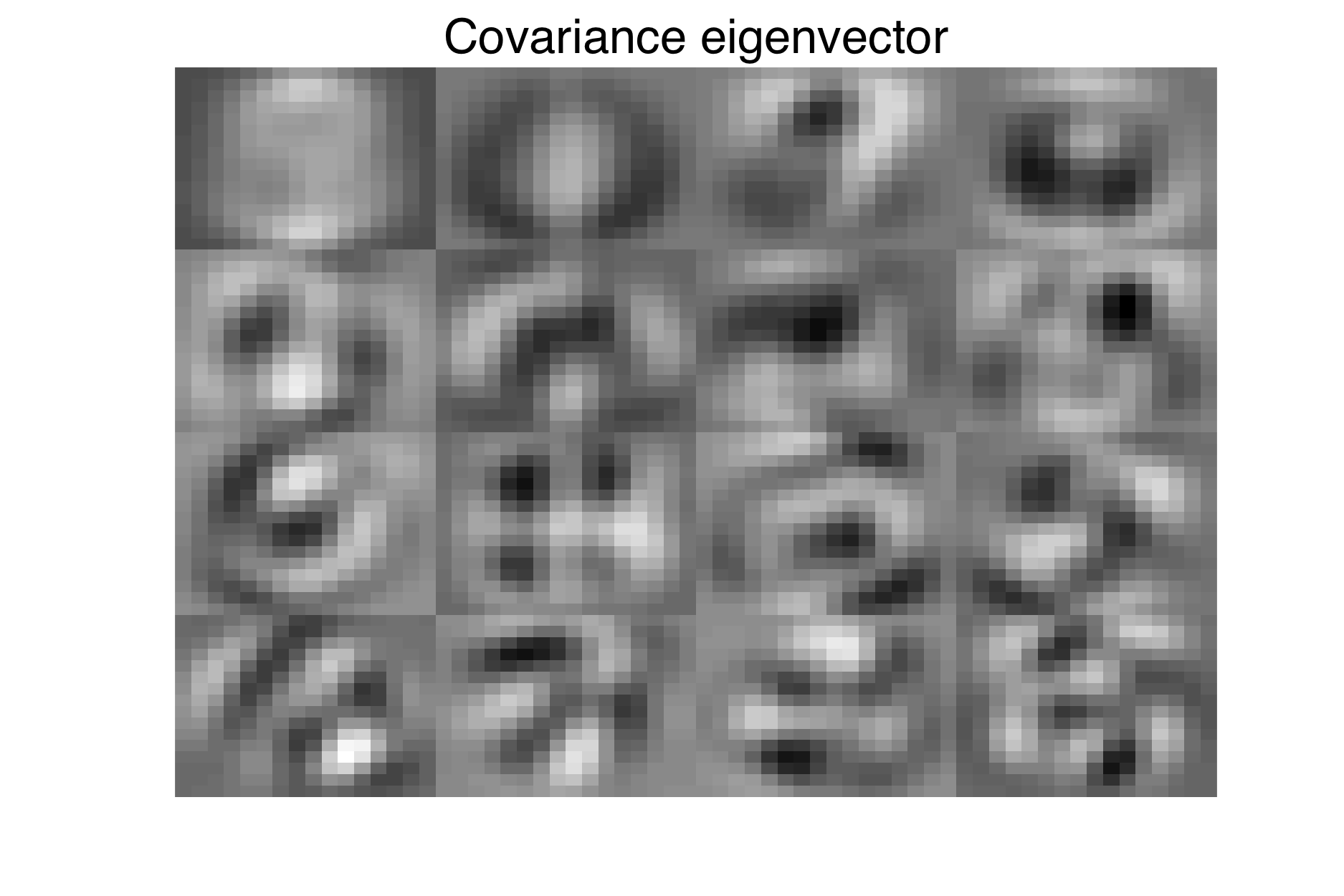} 
\includegraphics[width=0.3\textwidth]{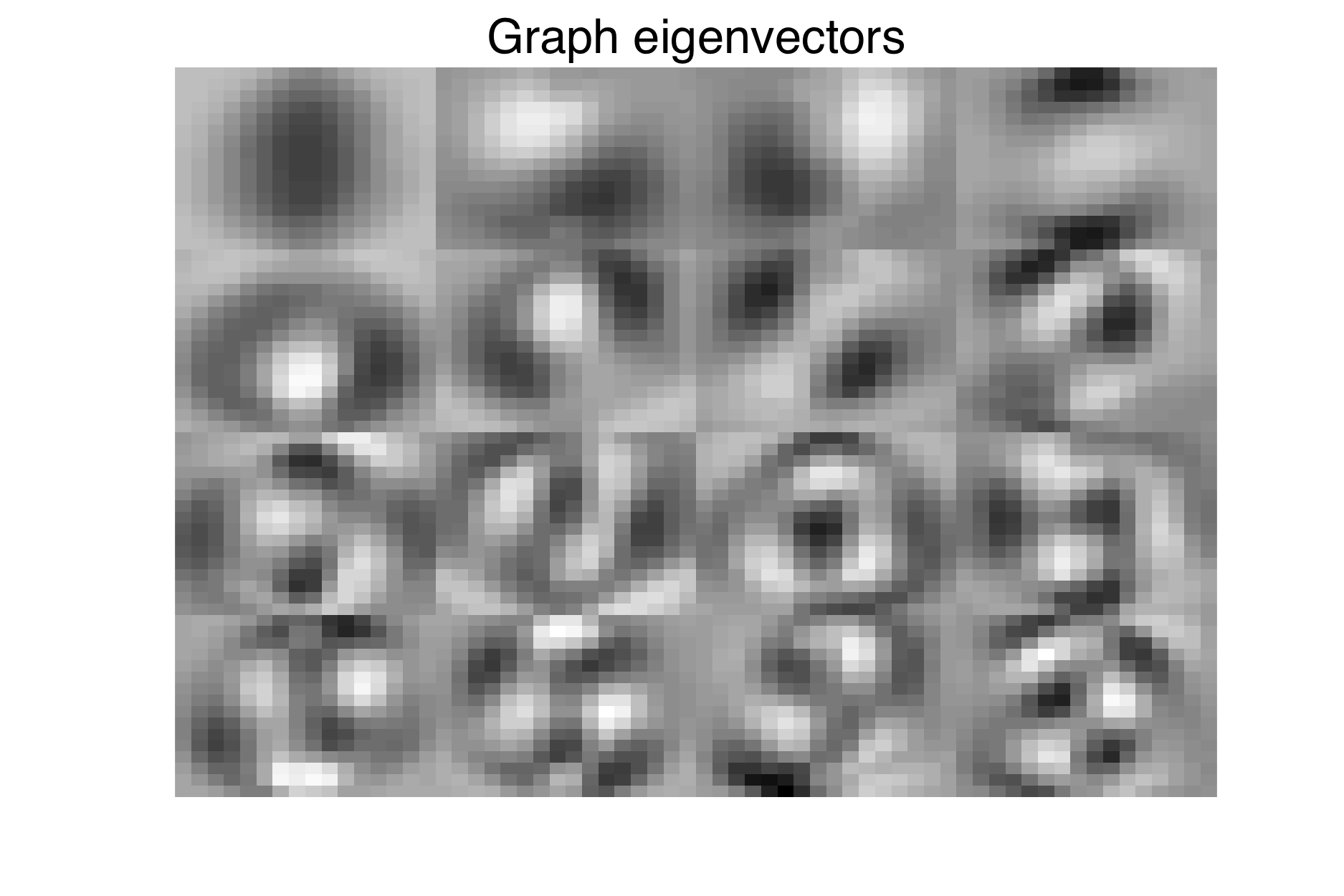} 
\includegraphics[width=0.3\textwidth]{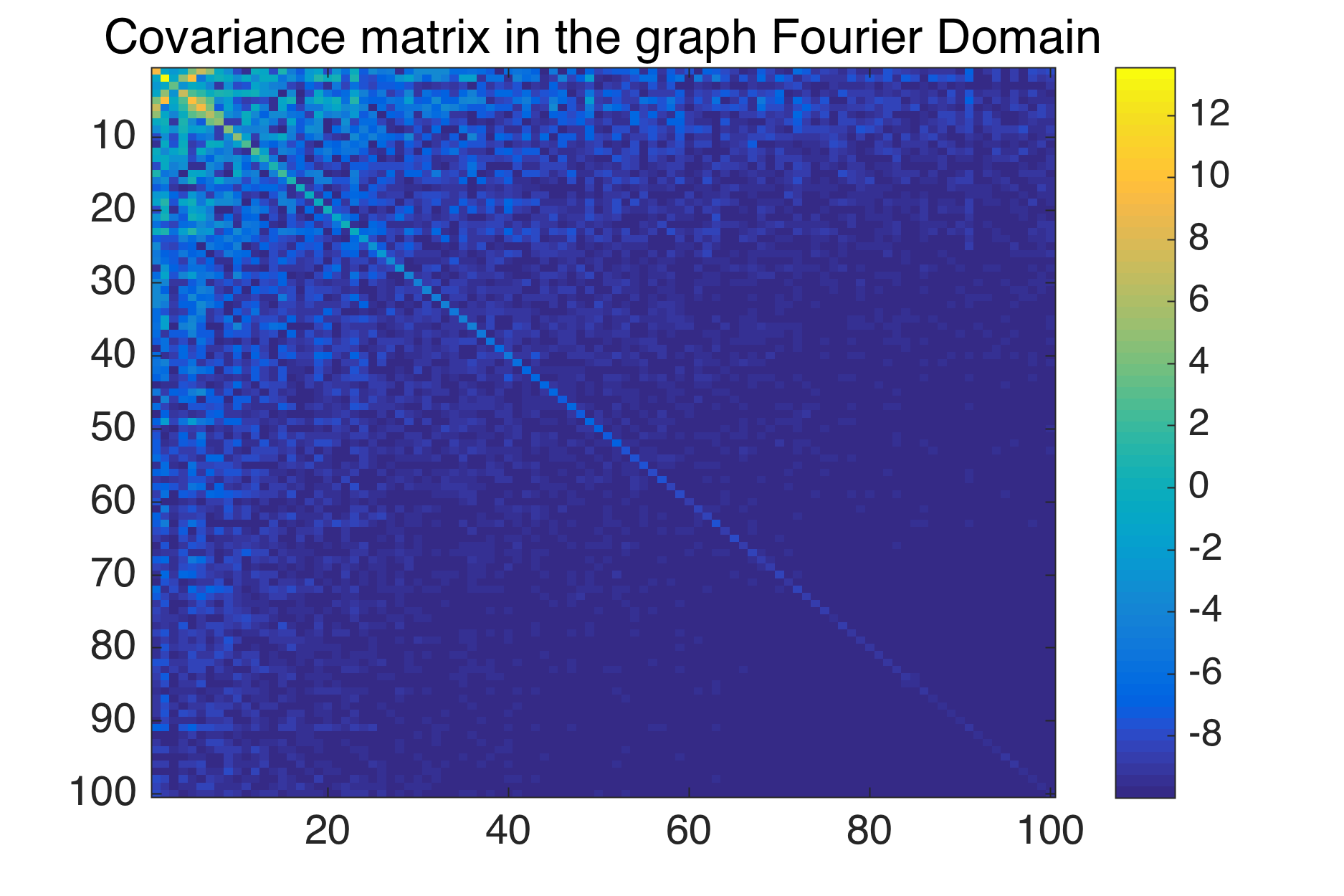}
\end{center}
\caption{Studying the full USPS dataset. Left: Covariance eigenvectors associated with the $16$ highest eigenvalues. Right: Laplacian eigenvectors associated to the $16$ smallest non-zero eigenvalues. Because of stationarity, Laplacian eigenvectors are similar to the covariance eigenvectors. Right: $\Gamma_r= P^\top C_r P$ in dB. Note the diagonal shape of the matrix implying that $P$ is aligned with the eigenvectors of $C_r$. (Figure taken from \cite{2016arXiv160102522P} with permission).}
\label{fig:exp4}
\end{figure}

From this example, we can conclude that every column of a  low-rank matrix $X$ lies approximately in the span of the eigenvectors $P_{k_r}$ of the features graph $G_r$, where $k_r$ denotes the eigenvectors corresponding to the smallest $k_r$ eigenvalues. This is similar to PCA, where a low-rank matrix is represented in the span of the first few principal directions or atoms of the basis. Alternately,  the Laplacian eigenvectors are meaningful features for the USPS dataset. Let the eigenvectors $P$ of $\Larg_r$ be divided into two sets $(P_{k_r} \in \Re^{p\times k_r}, \bar{P}_{k_r} \in \Re^{p \times (p - k_r)})$. Then, we can write

\begin{equation} \label{eq:P}
\Larg_r = P\Lambda_r P^\top = P_{k_r}\Lambda_{k_r}P^{\top}_{k_r} + \bar{P}_{k_r}\bar{\Lambda}_{k_r}\bar{P}^{\top}_{k_r}.
\end{equation}

 Note that the columns of $P_{k_r}$ contain the eigenvectors corresponding to the low graph frequencies and $\bar{P}_{k_r}$ contains those corresponding to higher graph frequencies. Now we can write, $X = X^{*} + E$, where $X^{*}$ is the low-rank part and $E$ models the noise or corruptions. Thus,
$$X = P_{k_r} A + \bar{P}_{k_r}\bar{A} ~ ~\text{and} $$  $$X^{*} =  P_{k_r} A  $$
where $A\in \Re^{k_r \times n}$ and $\bar{A} \in \Re^{(p - k_r) \times n}$. From Fig.~\ref{fig:exp3} it is also clear that $\|\bar{P}_{k_r}\bar{A}\|_F \ll \|P_{k_r} A\|_F$ for a specific value of $k_r$.

\subsubsection{The graph of samples provides an embedding}
Let  $\Larg_c \in \mathbb{R}^{n \times n} = D_{c}^{-1/2}(D_c - W_c)D_{c}^{-1/2} $ as the normalized graph Laplacian for the graph $G_c$ between columns of the data $Y$. The smallest eigenvectors of the graph of samples (columns) $G_c$ provide an embedding of the data in the low-dimensional space \cite{belkin2003laplacian}.  This is the heart of  many algorithms in clustering \cite{ng2002spectral} and dimensionality reduction \cite{belkin2003laplacian} and has a similar interpretation as the principal components in PCA. Let $C_c$ be the sample covariance of the data $Y$, $\Larg_c = Q\Lambda_c Q^\top$ and $\Gamma_c = Q^\top  C_c Q$, then we can define the \textit{alignment order} ${s}_c(\Gamma_c)$ and \textit{rank-k alignment order} $\hat{s}_c(\Gamma_c)$ similar to that for the rows of the data $Y$ and derive the same conclusions. We do not repeat this procedure here for brevity. 

Thus, we argue that every row of a  low-rank matrix lies in  the span of the first few eigenvectors of the graph of samples $G_c$. In our present application this term has two effects.
\begin{enumerate}
\item Firstly, when the data has a class structure, the graph of samples enforces the low-rank $X$ to benefit from this class structure. This results in an enhanced clustering of the low-rank signals.
\item Secondly, it will force that the low-rank $X$  of the signals is well represented by the first few Laplacian eigenvectors associated to low $\lambda_{cj}$. 
\end{enumerate}  

Let the eigenvectors $Q$ of $\Larg_c$ be divided into two sets $(Q_{k_c} \in \Re^{n\times k_c}, \bar{Q}_{k_c} \in \Re^{n \times (n - k_c)})$, where $k_c$ denotes the eigenvectors in $Q$ corresponding to the smallest $k_c$ eigenvalues. Then, we can write 

\begin{equation}\label{eq:Q}
\Larg_c = Q\Lambda_c Q^\top = Q_{k_c}\Lambda_{k_c}Q^{\top}_{k_c} + \bar{Q}_{k_c}\bar{\Lambda}_{k_c}\bar{Q}^{\top}_{k_c}.
\end{equation}

Note that the columns of $Q_{k_c}$ contain the eigenvectors corresponding to the low graph frequencies and $\bar{Q}_{k_c}$ contains those corresponding to higher graph frequencies. Now, we can write:
$$X = BQ^{\top}_{k_c} + \bar{B}\bar{Q}^{\top}_{k_c}  ~ ~\text{and} $$ $$X^{*} = BQ^{\top}_{k_c}$$
where $B\in \Re^{p \times k_c}$ and $\bar{B} \in \Re^{p \times (n - k_c)}$. As argued in the previous subsection, $\|\bar{B}\bar{Q}^{\top}_{k_c}\|_F   \ll  \| BQ^{\top}_{k_c} \|_F$.

\subsubsection{Real world examples}
Many real world datasets can benefit from such notions of data similarity. We already described an example for the USPS digits. Other examples include the matrix of the users and items in a typical recommendation system where one can assume that the users and items form communities among themselves. Thus it would be meaningful to exploit the graph similarity between the users and another between the items. One can also use such a notion of similarity to extract the repeating patterns  across space and time such as in EEG, FMRI and other types of the brain data.

\subsubsection{Our proposal}
Thus we propose to construct two graphs, one between the rows and another between the columns of the data. Clearly our proposed framework is moving in a direction where we want to exploit the row and column graph based similarity of the data matrix to extract the low-rank representation. In simpler words, we look for the underlying noiseless low-rank data matrix whose rows vary smoothly on the row graph and whose columns vary smoothly on the column graph.

Smoothness in general corresponds to the low frequency content of a signal. For example, for our  database of the noisy images of a landscape, the smooth images will be the clean landscape which constitutes most of the information in the dataset. In other words, a smooth signal corresponds to a representative signal which explains most of the data variance or which correspond to the low-frequency components of the data. Now, how can we characterize the notion of smoothness on a graph in a formal manner? This motivates us to define the low-rank matrices on graphs.

\subsection{Low-rank Matrices on Graphs (LRMG): The Definition}\label{sec:def_lrmg}
From the above explanation related to the role of the two graphs, we can conclude the following facts about the representation of any clusterable low-rank  matrix $X^* \in \Re^{p \times n}$ ($p$ features and $n$ samples).
\begin{enumerate}
\item Each of its columns can be represented as the span of the  Laplacian eigenvectors of the graph of features $G_r$, i.e, $X^* = P_{k_r} A $.
\item Each of its rows can  be represented as a span of the Laplacian eigenvectors of the graph of samples, i.e, $X^* = BQ^{\top}_{k_c} $. 
\end{enumerate}

As already pointed out, only the first $k_c$ or $k_r$ eigenvectors of the graphs correspond to the low frequency information, therefore, the other eigenvectors correspond to high frequency content. We are now in a position to define \textit{low-rank matrix on graphs}.

\begin{defn}
A matrix $X^*$ is $(k_r, k_c)$-low-rank on the graphs $\Larg_r$ and $\Larg_c$ if $(X^*)_i^\top \in {\rm span}(Q_{k_c})$ for all $i = 1, \ldots, p$, and $(X^*)_j \in {\rm span}(P_{k_r})$ for all $j = 1, \ldots, n$. The set of  $(k_r, k_c)$-low-rank matrices on the graphs $\Larg_r$ and $\Larg_c$ is denoted by $\mathcal{LR}( P_{k_r}, Q_{k_c})$.
\end{defn}

We note here that $X^* \in {\rm span}(P_{k_r})$ means that the columns of $X^*$ are in ${\rm span}(P_{k_r})$, i.e, $(X^*)_i \in {\rm span}(P_{k_r})$, for all $i = 1, \ldots, n$, where for any matrix $A$, $(A)_i$ is its $i^{th}$ column vector.

\section{Recovering Low-Rank Matrices on Graphs}\label{sec:frpcag}

\begin{mdframed}[style=MyFrame]
\begin{center}
\textbf{RECOVERY:  \textit{\nauman{``Low-rank matrices on graphs (LRMG) can be recovered via dual graph regularization with a suitable loss function. This method is called Fast Robust PCA on graphs (FRPCAG)''}}}. 
\end{center}
\end{mdframed}

We want to recover a matrix $X$ that is smooth / low-rank with respect to the row and column graph of the data $Y$. Furthermore, we want it to be robust to noise and a wide variety of gross errors in the data. Thus, we solve the following generalized robust low-rank recovery problem 
\begin{align}\label{eq:frpcag1}
& \min_{{X}} \Phi({X} - {Y}) \nonumber\\
& \text{s.t:} \hspace{0.1cm} {X} \in span(P_{k_r}), \hspace{0.2cm} {X}^\top \in span(Q_{k_c})
\end{align}
where $\Phi(X - Y)$ models a smooth or a non-smooth loss function depending on the type of noise or errros in the dataset. We provide a few examples below.
\begin{enumerate}
\item $\|X-Y\|_1$, the $L_1$ loss function. This povides robustness to sparse gross errors in the dataset.
\item $\|X-Y\|^{2}_F$, the $L_2$ loss function which provides robustness to Gaussian noise.
\item $\|X-Y\|_{2,1}$, the $L_{2,1}$ loss function which provides robustness to sample specific corruptions and outliers, where $L_{2,1}$ is the mixed norm which forces the columns of a matrix to be zero. This promotes group sparsity along the columns of a matrix.
\end{enumerate}
In the remainder of this paper we discuss only about~\eqref{eq:frpcag1} with $L_1$ loss function. However, our theoretical results are general and can be easily extended to the case of more specific loss functions.

Eq.~\ref{eq:frpcag1} is computationally expensive to solve because it requires the information about $P_{k_r}$ and $Q_{k_c}$ which can be obtained by diagonalizing the Laplacians ${\Larg}_r$ and ${\Larg}_c$ and cost $\mathcal{O}(p^3)$, $\mathcal{O}(n^3)$ respectively. Therefore, we need to transform the constraints such that we do not require the diagonalization of the Laplacians. Given a vector $x$, its smoothness on a graph Laplacian $\Larg$ can be measured by using the graph Tikhonov / graph dirichlet energy $x^\top \Larg x$ (as mentioned in Section \ref{sec:graphs}). The lower is this energy, more smooth / low-rank is the signal $x$. Thus we can transform the problem~\eqref{eq:frpcag1} to the following:
\begin{align}\label{eq:gfrpcag}
& \min_{{X}} \Phi(X - Y) + \gamma_c\tr({X}{\Larg}_c{X}^\top) + \gamma_r\tr({X}^\top {\Larg}_r {X}),
\end{align}

where $\gamma_c$ and $\gamma_r$ are parameters which control the regularization of the two graph tikhonov terms.  These constants trade-off the amount of signal energy corresponding to the high frequency content versus the low frequency content. Thus, these constants indirectly control the rank of $X$.

\subsection{Fast Robust PCA on Graphs (FRPCAG) \& its Optimization Solution}\label{sec:optimization_frpcag}
Using the $L_1$ loss function in~\eqref{eq:frpcag1} we get: 
\begin{align}\label{eq:frpcag}
& \min_{{X}} \|{X} - {Y}\|_1 + \gamma_c\tr({X}{\Larg}_c{X}^\top) + \gamma_r\tr({X}^\top {\Larg}_r {X})
\end{align}
Model~\eqref{eq:frpcag} is known as Fast Robust PCA on Graphs (FRPCAG) \cite{shahid2015fast} and corresponds to our earlier work in this direction. 

We use the Fast Iterative Soft Thresholding Algorithm (FISTA) \cite{beck2009fast} to solve problem~\eqref{eq:frpcag1}. Let $g: \mathbb{R^{N}}\rightarrow \mathbb{R}$ be a convex, differentiable function with a $\beta$-Lipschitz continuous gradient $\nabla g$ and $h: \mathbb{R^{N}}\rightarrow \mathbb{R}$ a convex function with a proximity operator $\prox_{h}:\mathbb{R}^N\rightarrow\mathbb{R}^N$ defined as:
\begin{equation*}
 \prox_{\lambda h}(y) = \argmin_s \frac{1}{2} \|s-y\|_2^2 + \lambda  h(s) .
\end{equation*}
Our goal is to minimize the sum $g(s)+h(s)$, which is done efficiently with proximal splitting methods. More information about proximal operators and splitting methods for non-smooth convex optimization can be found in \cite{combettes2011proximal}.
For model~\eqref{eq:frpcag1}, $g(X) = \gamma_{c}\tr(X\Larg_{c}X^\top ) + \gamma_{r}\tr(X^\top \Larg_{r}X)$ and $h(X) = \|Y-X\|_{1}$. The gradient of $g$ becomes
\begin{equation}\label{eq:grad}
\nabla_g(X) =  2( \gamma_{c} X\Larg_{c}  + \gamma_{r} \Larg_{r}X).
\end{equation}
We define an upper bound on the Lipschitz constant $\beta$ as $\beta \leq \beta' = 2\gamma_c \|\Larg_c\|_2 + 2\gamma_r \|\Larg_r\|_2$ where $\|\Larg \|_2$ is the  spectral norm (or maximum eigenvalue) of $\Larg$. Moreover, the proximal operator of the function $h$ is the $\ell_1$ soft-thresholding given by the elementwise operations (here $\circ$ is the Hadamard product)
\begin{equation}\label{eq:prox}
\prox_{\lambda h }(X) = Y + \sign(X-Y) \circ \max (|X-Y|-\lambda ,0).
\end{equation}
The FISTA algorithm \cite{beck2009fast} can now be stated as Algorithm \ref{CHalgorithm},
\begin{algorithm}
\caption{FISTA for FRPCAG}
\label{CHalgorithm}
\begin{algorithmic}
\State INPUT: $S_1 = Y$, $X_0 = Y$, $t_1 = 1$, $\epsilon > 0$
\For{ $j = 1,\dots J$ }
\State $X_{j} = \prox_{\lambda_{j}h}(S_{j}-\lambda_{j}\nabla g(S_{j}))$
\State $t_{j+1} = \frac{1+\sqrt{1+4t_j^2}}{2}$
\State $S_{j+1} = X_j +\frac{t_j-1}{t_{j+1}} (X_j-X_{j-1})$
\If{$\|S_{j+1} - S_{j}\|_F^2 < \epsilon \| S_{j}\|_F^2$}
\State BREAK
\EndIf
\EndFor
\State OUTPUT: $X_{j+1}$
\end{algorithmic}
\end{algorithm}
where $\lambda$ is the step size (we use $\lambda = \frac{1}{\beta '}$), $\epsilon$ the stopping tolerance and $J$ the maximum number of iterations.

While \cite{shahid2015fast} focuses more on the experimental results for FRPCAG, one of the main contributions of this work is to present a theoretical understanding of this model in a more involved manner. Furthermore, we also present a more generalized version of FRPCAG in the discussion that follows. Therefore, we take a step back here and perform a theoretical underpinning of the model~\eqref{eq:frpcag1}. It is this analysis which reveals that the model recovers an approximate low-rank representation.  First, we present our first main result in simple words here:

\section{Theoretical Analysis of FRPCAG}\label{sec:theory}
\subsection{A Summary of the analysis}
\begin{mdframed}[style=MyFrame]
\begin{center}
\textbf{\textit{\nauman{``FRPCAG \eqref{eq:frpcag1}  gives an approximate low-rank representation $X$ of a LRMG $Y$ which is $k_c, k_r$ clusterable across its columns and rows. The left and right singular vectors of the low-rank matrix $X$ are given by subspace rotation operation and the singular values are penalized by the graph eigenvalues. The approximation error of $X$ depends on the spectral gaps of the Laplacians $\Larg_r$ and $\Larg_c$ respectively, where the spectral gaps are defined as the ratios $\lambda_{k_r}/\lambda_{k_r + 1}$ and $\lambda_{k_c}/\lambda_{k_c + 1}$ ''}}}. 
\end{center}
\end{mdframed}

\subsection{Main Theorem}
 Now we are ready to formalize our findings mathematically and prove that any solution of \eqref{eq:frpcag1} yields an approximately low-rank matrix. In fact, we prove this for  any proper, positive, convex and lower semi-continuous loss function $\phi$ (possibly $\ell_p$-norms $\|\cdot\|_1$, $\|\cdot\|_2^2$, ..., $\|\cdot\|_p^p$). We re-write \eqref{eq:frpcag1} again with a general loss function $\phi$
\begin{equation}\label{eq:optim}
\min_{X} \phi(Y - X) + \gamma_c \tr(X \Larg_c X^\top) + \gamma_r \tr(X^\top \Larg_r X)
\end{equation}

 Before presenting our mathematical analysis we gather a few facts which will be used later:
\begin{itemize}
\item We assume that the observed data matrix $Y$ satisfies $Y = Y^* + E$ where $Y^* \in \mathcal{LR}(P_{k_r},Q_{k_c})$ and $E$ models noise/corruptions. Furthermore, for any $Y^* \in \mathcal{LR}(P_{k_r},Q_{k_c})$ there exists a matrix $C$ such that $Y^* = P_{k_r} C Q_{k_c}^\top$.
\item $\Larg_c = Q\Lambda_c Q^\top = Q_{k_c} \Lambda_{k_c} Q^{\top}_{k_c} + \bar{Q}_{k_c} \bar{\Lambda}_{k_c} \bar{Q}^{\top}_{k_c} $, where $\Lambda_{k_c} \in \Re^{k_c \times k_c}$ is a diagonal matrix of lower eigenvalues  and    $\bar{\Lambda}_{k_c} \in \Re^{(n - k_c) \times (n - k_c)}$ is also a diagonal matrix of higher graph eigenvalues. All values in $\Lambda_c$ are sorted in increasing order, thus $0  = \lambda_0 \leq \lambda_1 \leq \cdots \leq \lambda_{k_c} \leq \cdots \leq \lambda_{n-1} $. The same holds for $\Larg_r$ as well.
\item For a $\K$-nearest neighbors graph constructed from a $k_c$-clusterable data (along samples / columns) one can expect $\lambda_{k_c}/\lambda_{k_c + 1} \approx 0$ as $\lambda_{k_c} \approx 0$ and $\lambda_{k_c} \ll \lambda_{k_c + 1}$. The same holds for the graph of features  / rows $\Larg_r$ as well.
\item For the proof of the theorem, we will use the fact that for any $Y \in \Re^{p \times n}$, there exist $A \in \Re^{k_r \times n}$ and $\bar{A} \in \Re^{(p-k_r) \times n}$ such that $X = P_{k_r} A + \bar{P}_{k_r} \bar{A}$, and $B \in \Re^{p \times k_c}$ and $\bar{B} \in \Re^{p \times (n-k_c)}$ such that $X = B Q_{k_c}^\top + \bar{B} \bar{Q}_{k_c}^\top$.
\end{itemize}

\begin{thm}\label{thm:gfrpcagnoisy} 
Let $Y^* \in \mathcal{LR}(P_{k_r}, Q_{k_c})$, $\gamma>0$, and $E \in \Re^{p \times n}$. Any solution $X^*  \in \Re^{p \times n}$ of \eqref{eq:optim} with $\gamma_c = \gamma/\lambda_{k_c+1}$, $\gamma_r = \gamma/\lambda_{k_r+1}$ and $Y = Y^* + E$ satisfies
\begin{align}\label{eq:bound}
& \phi(X^* - Y) + \gamma_c \| X^{*} \bar{Q}_{k_c}\|_F^2 + \gamma_r \|\bar{P}_{k_r}^\top  X^*\|_F^2  \leq \phi(E) + \gamma \|Y^*\|_F^2 \Big( \frac{\lambda_{k_c}}{\lambda_{k_c+1}} + \frac{\lambda_{k_r}}{\lambda_{k_r+1}} \Big).
\end{align}
where $\lambda_{k_c}, \lambda_{k_{c}+1} $ denote the $k_c, k_c + 1 $ eigenvalues of ${\Larg}_c$,  $\lambda_{k_r}, \omega_{k_{r}+1}$ denote the $k_r, k_r + 1 $ eigenvalues of ${\Larg}_r$.
\end{thm}
\begin{proof}
As $X^*$ is a solution of \eqref{eq:optim}, we have
\begin{align}\label{eq:optim_bound}
& \phi(X^* - Y) + \gamma_c \tr(X^* \Larg_c (X^*)^\top) + \gamma_r \tr((X^*)^\top \Larg_r X^*) \nonumber \\
 & \leq 
\phi(E) +  \gamma_c \tr(Y^* \Larg_c (Y^*)^\top) + \gamma_r \tr((Y^*)^\top \Larg_r Y^*).
\end{align}
Using the facts that $\Larg_c = Q_{k_c} \Lambda_{k_c} Q_{k_c}^\top + \bar{Q}_{k_c} \bar{\Lambda}_{k_c} \bar{Q}_{k_c}^\top$ and that there exists $B \in \Re^{p \times k_c}$ and $\bar{B} \in \Re^{p \times (n-k_c)}$ such that $X^* = B Q_{k_c}^\top + \bar{B} \bar{Q}_{k_c}^\top$, we obtain
\begin{align*}
& \tr(X^* \Larg_c (X^*)^\top) 
= \tr(B  \Lambda_{k_c}  B^\top) + \tr(\bar{B}  \bar{\Lambda}_{k_c} \bar{B}^\top)
\nonumber \\
& \geq \tr(\bar{\Lambda}_{k_c} \bar{B}^\top \bar{B}) \geq \lambda_{k_c+1}\|\bar{B}\|_F^2  = \lambda_{k_{c}+1} \|X^* \bar{Q}_{k_c}\|_F^2.
\end{align*}
Then, using the fact that there exists $C \in \Re^{k_r \times k_c}$ such that $Y^* = P_{k_r} C Q_{k_c}^\top$, we obtain
\begin{align*}
\tr(Y^* \Larg_c (Y^*)^\top) = \tr(C \Lambda_{k_c} C^\top) \leq \lambda_{k_c} \|C\|_F^2 = \lambda_{k_c} \|Y^*\|_F^2.
\end{align*}
Similarly, we have
\begin{align*}
\tr((X^*)^\top \Larg_r X^*) 
\geq
\omega_{k_r+1} \|\bar{P}_{k_r}^\top X^*\|_F^2,
\end{align*}
\begin{align*}
\tr((X^*)^\top \Larg_2 X^*) \leq \omega_{k_2} \|X^*\|_F^2.
\end{align*}
Using the four last bounds in \eqref{eq:optim_bound} yields
\begin{align*}
& \phi(X^* - Y) + \gamma_c \lambda_{k_c+1} \|X^* \bar{Q}_{k_c}\|_F^2 + \gamma_r \omega_{k_r+1} \|\bar{P}_{k_r}^\top X^*\|_F^2 
\leq \nonumber \\
& \phi(E) +  \gamma_c \omega_{k_c} \|Y^*\|_F^2 + \gamma_r \omega_{k_r} \|Y^*\|_F^2,
\end{align*}
which becomes 
\begin{align*}
& \phi(X^* - Y) + \gamma \|X^* \bar{Q}_{k_c}\|_F^2 + \gamma \|\bar{P}_{k_r}^\top X^*\|_F^2 \nonumber \\
& \leq 
\phi(E) +  \gamma \|Y^*\|_F^2 \left( \frac{\lambda_{k_c}}{\lambda_{k_c+1}}  + \frac{\omega_{k_r}}{\omega_{k_r+1}} \right)
\end{align*}
for our choice of $\gamma_c$ and $\gamma_r$. This terminates the proof.
\end{proof}

\subsection{Remarks on the theoretical analysis}
\eqref{eq:bound} implies that
\begin{align*}
& \|X^* \bar{Q}_{k_c}\|_F^2 + \|\bar{P}_{k_r}^\top U^*\|_F^2  \leq 
\frac{1}{\gamma} \phi(E) +  \|Y^*\|_F^2 \left( \frac{\lambda_{k_c}}{\lambda_{k_c+1}}  + \frac{\lambda_{k_r}}{\lambda_{k_r+1}} \right).
\end{align*}

The smaller $\|X^* \bar{Q}_{k_c}\|_F^2 + \|\bar{P}_{k_r}^\top X^*\|_F^2$ is, the closer $X^*$ to $\mathcal{LR}( P_{k_r},Q_{k_c})$ is. The above bound shows that to recover a low-rank matrix one should have large eigengaps ${\lambda_{k_c+1}}-{\lambda_{k_c}}$ and ${\lambda_{k_r+1}} - {\lambda_{k_r}}$. This occurs when the rows and columns of $Y$ can be clustered into $k_r$ and $k_c$ clusters. Furthermore, one should also try to chose a metric $\phi$ (or $\ell_p$-norm) that minimizes $\phi(E)$. Clearly, the rank of $X^{*}$ is approximately $\min\{k_r,k_c\}$.

\subsection{Why FRPCAG gives a low-rank solution? Implications of the Theoretical Analysis}\label{sec:theory_frpcag}

This section constitutes of a more formal and theoretical discussion of FRPCAG. We base our explanations on two arguments:
\begin{enumerate}
\item FRPCAG penalizes the the singular values of the data matrix. This can be viewed in two ways: 1) In the data domain via SVD analysis and 2) In the graph fourier domain \footnote{Here we assume that the data is stationary}.
\item FRPCAG is a subspace rotation / alignment method, where the left and right singular vectors of the resultant low-rank matrix are being aligned with the singular vectors of the clean data.
\end{enumerate}

\subsubsection{FRPCAG is a dual graph filtering / singular value penalization method} \label{sec:sqrt_penalization}

 In order to demonstrate how FRPCAG penalizes the  singular values of the data we study another  way to cater the graph regularization in the solution of the optimization problem which is contrary to the one presented in Section~\ref{sec:optimization_frpcag}. In Section \ref{sec:optimization_frpcag} we used a gradient for the graph regularization terms $\gamma_{c}\tr(X\Larg_{c}X^\top) + \gamma_r \tr(X^\top \Larg_{r}X)$ and used this gradient as an argument of the proximal operator for the soft-thresholding.  What we did not point out there was that the solution of the graph regularizations can also be computed by proximal operators. It is due to the reason that using proximal operators for graph regularization (that we present here) is more computationally expensive.   Assume that the prox of $\gamma_{c}\tr(X\Larg_{c}X^\top)$ is computed first, and let $Z$ be a temporary variable, then it can be written as: 
\begin{equation*}
\min_{Z} \|Y-Z\|^2_F + \gamma_{c}\tr(Z\Larg_{c}Z^\top)
\end{equation*}
The above equation has a closed form solution which is given as:
$$Z = Y(I + \gamma_c \Larg_c)^{-1}$$
Now, compute the proximal operator for the term $\gamma_{r}\tr(X^\top \Larg_{r}X)$
\begin{equation*}
\min_{X} \|Z-X\|^2_F + \gamma_{r}\tr(X^\top \Larg_{r}X)
\end{equation*}
The closed form solution of the above equation is given as:
$$X = (I + \gamma_r \Larg_r)^{-1} Z$$
Thus, the low-rank $U$ can be written as:
$$X = (I + \gamma_r \Larg_r)^{-1} Y (I + \gamma_c \Larg_c)^{-1} $$
after this the soft thresholding can be applied on $X$.

Let the SVD of $Y$, $Y = U_y\Sigma_y V^{\top}_y $, $\Larg_c = Q\Lambda_c Q^\top$ and $\Larg_r = P\Lambda_r P^\top$, then we get:
\begin{align*} 
 X  & = (I + \gamma_r P\Lambda_rP^\top)^{-1} U_y\Omega_y V^{\top}_y (I + \gamma_c Q\Lambda_c Q^\top)^{-1} \nonumber \\
 & = P(I + \gamma_r\Lambda_r)^{-1} P^\top U_y\Omega_y V^{\top}_y Q(I + \gamma_c \Lambda_c)^{-1} Q^\top
\end{align*}
thus, the singular values $\Sigma_y$ of $Y$ are penalized by $1/(1+\gamma_c \Lambda_{c})(1+\gamma_r \Lambda_{r})$. Clearly, the above solution requires the computation of two inverses which can be computationally intractable for big datasets. 

 
 The singular value thresholding effect is also shown in Fig.~\ref{fig:rotation_shrinking}, where the green vectors (scaled by their singular values) as learned via our model tend to shrink with the increasing regularization parameter. 
 
 \subsubsection{FRPCAG is a penalization method in the spectral graph fourier domain}
 Another way to analyse FRPCAG is to look at the penalization $\gamma_c\tr({X}{\Larg}_c{X}^\top) + \gamma_r\tr({X}^{\top}{\Larg}_c{X})$ in the spectral graph Fourier domain. 
As stated in Subsection~\ref{sec:meaning_of_xLx}, it is possible to rewrite it as a spectral penalization:
\begin{equation}
\gamma_c\tr({X}{\Larg}_c{X}^\top) + \gamma_r\tr({X}^{\top}{\Larg}_r{X}) =  \gamma_c \| \Lambda_c^{\frac{1}{2}} \hat{X} \|_F^2 + \gamma_r \| \Lambda_r^{\frac{1}{2}} \hat{X^\top} \|_F^2.
\end{equation}
The signal is thus penalized in the spectral graph domain with a weight proportional to the squared root of the graph eigenvalue. 
As a result, the signal is pushed towards the lowest frequencies of both graphs, enforcing its low-rank structure. This interpretation allows us to clearly identify the role of the regularization constant $\gamma_r$ and $\gamma_c$. With big constants $\gamma_c$ and $\gamma_r$, the resulting signal is closer to low rank but we also attenuate the content inside the band as illustrated in Figure~\ref{fig:rotation_shrinking}.

\subsubsection{FRPCAG is a weighted subspace alignment method}

 Let $X=U\Sigma V^{\top}$ be the SVD of $X$ and suppose now that we minimize w.r.t factors $U,\Sigma, V$ instead of $X$. Then we have
\begin{align}\label{eq:alignment}
\gamma_{c}\tr(X\Larg_{c}X^\top) + \gamma_{r}\tr(X^\top \Larg_{r}X)
= &  \gamma_{c}\tr(U\Sigma V^{\top}Q\Lambda_c Q^{\top}V\Sigma U^{\top}) + \gamma_{r}\tr(V\Sigma U^{\top}P\Lambda_r P^{\top}U\Sigma V^{\top}) \nonumber \\
= & \gamma_{c}\tr(\Sigma V^{\top}Q\Lambda_c Q^{\top}V\Sigma) + \gamma_{r}\tr(\Sigma U^\top P\Lambda_r P^{\top}U \Sigma)  \\
= & \sum_{i,j = 1}^{\min\{n,p\}}\sigma_{i}^{2}(\gamma_{c}\lambda_{cj} (v^{\top}_{i}q_j)^{2}+\gamma_{r}\lambda_{rj}(u^{\top}_{i}p_{j})^{2}) \nonumber \\
= & \sum_{i = 1}^{\min\{n,p\}}\sigma_{i}^{2}\left(\gamma_{c} \left(\sum_{j = 1}^{n}\lambda_{cj} (v^{\top}_{i}q_j)^{2}\right)+\gamma_{r}\left(\sum_{j = 1}^{p}\lambda_{rj}(u^{\top}_{i}p_{j})^{2}\right)\right), \nonumber
\end{align} 

where $\lambda_{cj}$ and $\lambda_{rj}$ are the eigenvalues in the matrices $\Lambda_c$ and $\Lambda_r$ respectively. The second step follows from $V^{\top}V = I$ and $U^\top U = I$ and the cyclic permutation invariance of the trace. In the standard terminology $v_i$ and $u_{i}$ are the principal components and principal directions of of the low-rank matrix $X$. From the above expression, the minimization is carried out with respect to the singular values $\sigma_i$ and the singular vectors $v_i, u_i$. The minimization has the following effect:
\begin{enumerate}
\item Minimize $\sigma_i$ by performing a penalization with the graph eigenvalues as explained earlier.
\item When $\sigma_i$ is big, the principal components $v_i$ are more well aligned with the graph eigenvectors $q_j$ for small values of $\lambda_{cj}$, i.e, the lower graph frequencies of $\Larg_c$ as compared to the $q_j$ for higher $\lambda_{cj}$. The principal directions $u_i$ are also more aligned with the graph eigenvectors $p_j$ for small values of $\lambda_{rj}$, i.e, the lower graph frequencies of $\Larg_r$ as compared to higher frequencies. This alignment makes sense as the higher eigenvalues correspond to the higher graph frequencies which constitute the noise in data.
\end{enumerate}

\subsubsection{FRPCAG is a subspace rotation method}

The coherency of the principal directions and components with their respective graph eigenvectors implicitly promotes the coherency  of these vectors with the principal components and directions of the clean data. This is of course due to the fact that we want $X^* \in \mathcal{LR}( P_{k_r}, Q_{k_c})$, assuming $X^*$ is the optimal solution. Note that this implication of FRPCAG is not different from the motivation of Section \ref{sec:motivation} as the eigenvectors of $C_r$ are the principal directions of $Y$ and the eigenvectors of $C_c$ are the principal components of $Y$. Thus the parameters $\gamma_c$ and $\gamma_r$ act as a force which pushes the singular vectors (principal directions and components) of the low-rank matrix away from the graph eigenvectors towards the singular vectors of the clean data. This is effect is shown in Fig.~\ref{fig:rotation_shrinking}. We call this effect as the subspace rotation / alignment effect and this is the key feature of our graph based low-rank matrix approximation.   

\begin{figure*}[htbp]
    \centering
        \centering
        \includegraphics[width=0.9\textwidth]{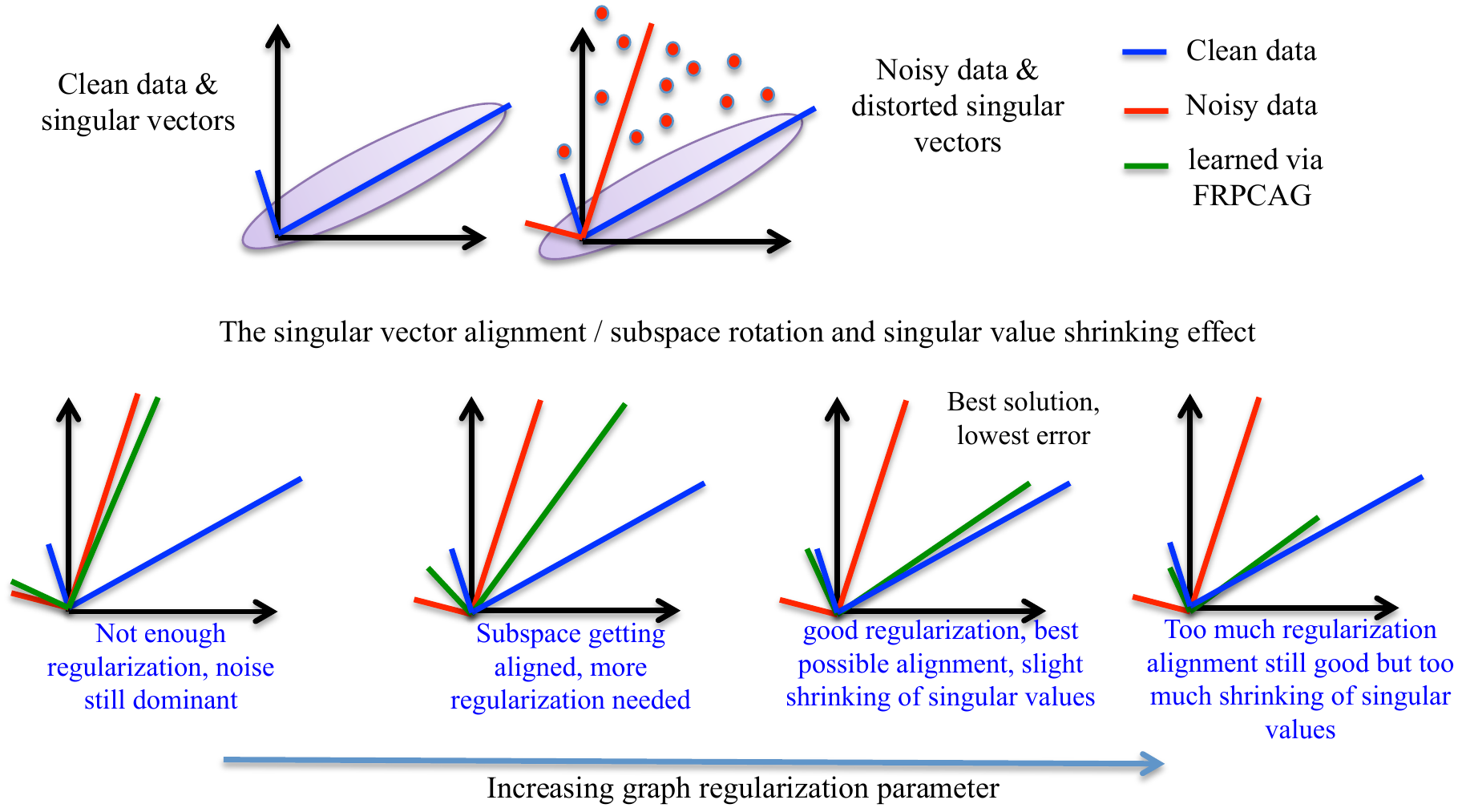}
         \caption{The alignment of the singular vectors with the singular vectors of clean data and the corresponding shrinking of the singular values with the increasing graph regularization parameter. The first row shows the original and noisy data, where the blue vectors show the first two singular vectors of the clean data and the red ones for the noisy data. The second row shows how the singular vectors (green) learned with the FRPCAG~\eqref{eq:frpcag} are being aligned with the singular vectors of the clean data. All the singular vectors are scaled by their singular values. Too much regularization shrinks the singular vectors.}
        \label{fig:rotation_shrinking}
    \end{figure*}
    
\subsubsection{FRPCAG is a weighted \textit{rank-k alignment order} maximizer}
Eq. \eqref{eq:alignment} can be further written as:
\begin{align*}
& = \gamma_c \tr(Q\Lambda_c Q^\top V \Sigma^2 V^\top) + \gamma_r \tr(P \Lambda_r P^\top U \Sigma^2 U^\top) \nonumber \\
& = \gamma_c \tr(Q \Lambda_c Q^\top C_c) + \gamma_r \tr(P \Lambda_r P^\top C_r) \nonumber \\
& = \gamma_c \tr(\Lambda_c Q^\top C_c Q)  + \gamma_r \tr(\Lambda_r  P^\top C_r P ) \nonumber \\
& = \gamma_c \tr(\Lambda_c \Gamma_c) + \gamma_r \tr(\Lambda_r \Gamma_r) \nonumber \\
& = \gamma_c \sum_i \lambda_{ci} \Gamma_{c i,i} +  \gamma_c \sum_j \lambda_{ri} \Gamma_{r j,j},
\end{align*}

where $C_c, C_r$ are the sample and feature covariance matrices and $\Gamma_r, \Gamma_c$ are the featutre and sample alignment orders as defined in Section \ref{sec:motivation}. Assuming the entries of $\Lambda_r, \Lambda_c$ are sorted in the increasing order, the last equation above corresponds to a weighted minimization of the diagonal entries of $\Gamma_c$ and $\Gamma_r$. Thus, the lower diagonal entries undergo more penalization than the top ones.  Now assume that there is a spectral gap, i.e, $\lambda_{c k_c} \ll \lambda_{c k_c + 1}$ and $\lambda_{r k_r} \ll \lambda_{r k_r + 1}$, then the minimization above leads to a higher penalization of all the entries above $k_r$ in $\Gamma_r$ and $k_c$ in $\Gamma_c$. This automatically leads to a maximixzation of the \textit{rank-k alignment orders} $\hat{s}_r (\Gamma_r), \hat{s}_c (\Gamma_c)$ (eq. \eqref{eq:rankkalign}).

\subsubsection{FRPCAG is a Bi-clustering method}
In the biclustering problem, one seeks to simultaneously cluster samples and features. Biclustering has applications in a wide variety of domains, ranging from text mining to collaborative filtering \cite{chi2014convex}. Very briefly, a typical bi-clustering method reveals:
\begin{enumerate}
\item the clusters across the rows and columns of a data matrix.
\item  the checkerboard pattern hidden in a noisy dataset.
\end{enumerate}

The theoretical analysis of FRPCAG reveals that the method can recover low-rank representation for a dataset $Y$ that is $k_r, k_c$ clusterable across its rows / features and columns / samples. This is not strange, as we show in the netx subsection that FRPCAG is able to successfully recover meaningful clusters not only acorss across samples but features as well.

\section{Putting all the pieces together: Working Examples of FRPCAG}\label{sec:examples}
\begin{mdframed}[style=MyFrame]
\begin{center}
\textbf{\textit{\nauman{``We justify the singular value penalization, singular vector alignment, low-rankness and bi-clustering effects of FRPCAG with a few artificial and real world datasets.''}}}. 
\end{center}
\end{mdframed}

We presented several different implications of FRPCAG in the previous section. These implications are justified with some artificial and real world datasets in this section.
\subsection{Low-rank recovery, singular value penalization and subspace alignment}

We generate a random low-rank matrix (rank = 10) of size $300 \times 300$ and construct a 10-nearest neighbors graph between its rows $G_r$ and columns $G_c$.  Let $\Larg_c = Q\Lambda_c Q^\top$, $\Larg_r = P\Lambda_r P^\top$. Then, we generate a noiseless data matrix $Y_0$ as $P_{10} R Q^{\top}_{10}$, where $P_{10}, Q_{10}$ are the first 10 vectors in $P,Q$ and $R \in \Re^{10 \times 10}$ is a random Gaussian matrix. Then, we corrupt $Y_0$  with Gaussian noise $\mathcal{N}(0,\sigma^2 I)$ to get $Y$. Finally we perform FRPCAG on noisy $Y$ to recover the low-rank for three different regularization parameter settings: $\gamma_r = \gamma_c = 0.01, 1 , 100$. 

Let $X = U\Sigma V^\top$ be the SVD of low-rank $X$ and $Y_0 = U_0 \Sigma_0 {{V}^\top}_0$ be the SVD of clean data. To study the error of approximation $\|Y_0 - X\|^{2}_F/\|Y_0\|^{2}_F$, we plot the low-rank $X$, $\Gamma_r, \Gamma_c$, the \textit{alignment order} $s_r(\Gamma_r)$,  $ s_c(\Gamma_c)$ and \textit{rank-k alignment} $\hat{s}_r(\Gamma_r),\hat{s}_c(\Gamma_c)$ ($k = 10$), the singular values $\Sigma$ of $X$, the coherency between  $(U,U_0)$, $(V,V_0)$ in Fig.~\ref{fig:demo_rotation_new}. Following conclusions can be drawn:
\begin{enumerate}
\item The lowest error occurs for $\gamma_r = \gamma_c = 1$ (it is possible to get a lower error with finer parameter tuning but the purpose here is to study the effect of regularization only).
\item For this parameter setting, the singular values $\Sigma$ decay to 0 after 10.
\item An interesting observation is the coherency between $(U,U_0)$ and $(V,V_0)$ for increasing regularization. Clearly, for small parameters, these pairs are incoherent. The $1^{st}$ 9 vectors of $U$ and $V$ are aligned with the $1^{st}$ 9 vectors of $U_0$ and $V_0$ for the best parameter $\gamma_c = \gamma_r = 1$.  The coherency of the pairs of singular vectors decreases again for higher parameters and the thresholding on the singular values is more than required. The \textit{rank-k alignment} order also shows a very high degree of alignment between the graph eigenvectors $P,Q$ and covariance matrices $C_r,C_c$.
\end{enumerate}

\begin{figure*}[htbp]
    \centering
        \centering
        \includegraphics[width=0.7\textwidth]{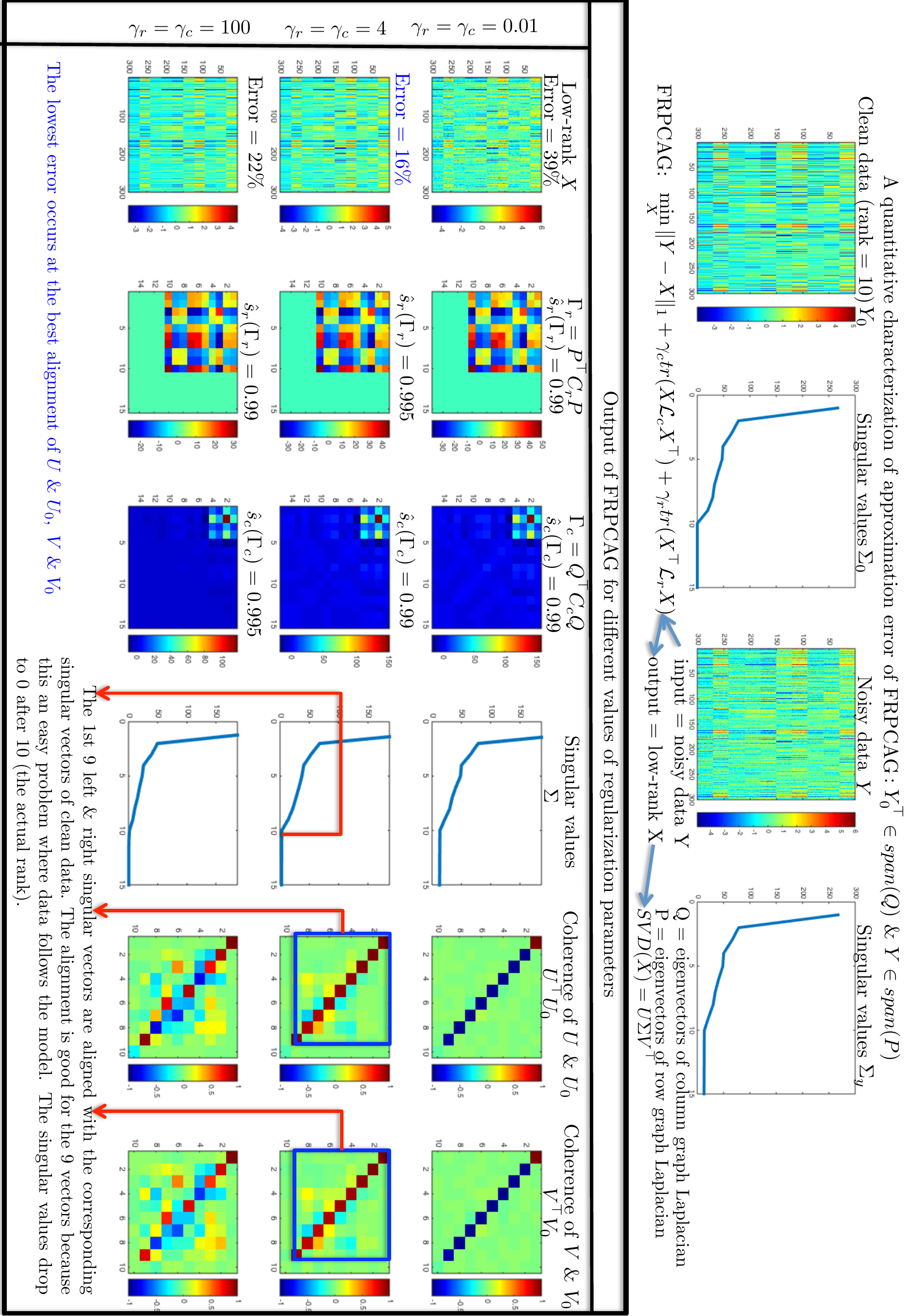}
         \caption{A study of the error of approximation $\|Y_0 - X\|^{2}_F/\|Y_0\|^{2}_F$ for a data generated from graph eigenvectors. We plot the low-rank $X$, the \textit{alignment order} $s_r(\Gamma_r), s_c(\Gamma_c)$ and \textit{rank-k alignment} $\hat{s}_r(\Gamma_r),\hat{s}_c(\Gamma_c)$ ($k = 10$), the singular values $\Sigma$ of $X$, the coherency between  $(U,U_0)$, $(V,V_0)$}
        \label{fig:demo_rotation_new}
    \end{figure*}

    For the best parameters $\gamma_r = \gamma_c = 1$, there is a strong alignment between the $1^{st}$ 9 singular vectors. Furthermore, the thresholding on the $10^{th}$ singular value makes it slightly lower than that of the clean data $Y_0$ which results in a  miss-alignment of the $10^{th}$ singular vector.  

\subsection{Clustering, singular value penalization and subspace alignment}
The purpose of this example is to show the following effects of FRPCAG:
\begin{enumerate}
\item The model recovers a close-to-low-rank representation.
\item The principal components $V$ and principal directions $U$ of $X$ (assuming $X = U\Sigma V^\top$) align with the first few eigenvectors of their respective graphs, automatically revealing a low-rank and enhanced class structure.
\item  The singular values of the low-rank matrix obtained using our model closely approximate those obtained by nuclear norm based models even  in the presence of corruptions.
\end{enumerate}

Our justification relies mostly on the quality of the singular values of the low-rank representation and the alignment of the singular vectors with their respective graphs.

 \begin{figure*}[htbp]
    \centering
        \centering
        \includegraphics[width=0.9\textwidth]{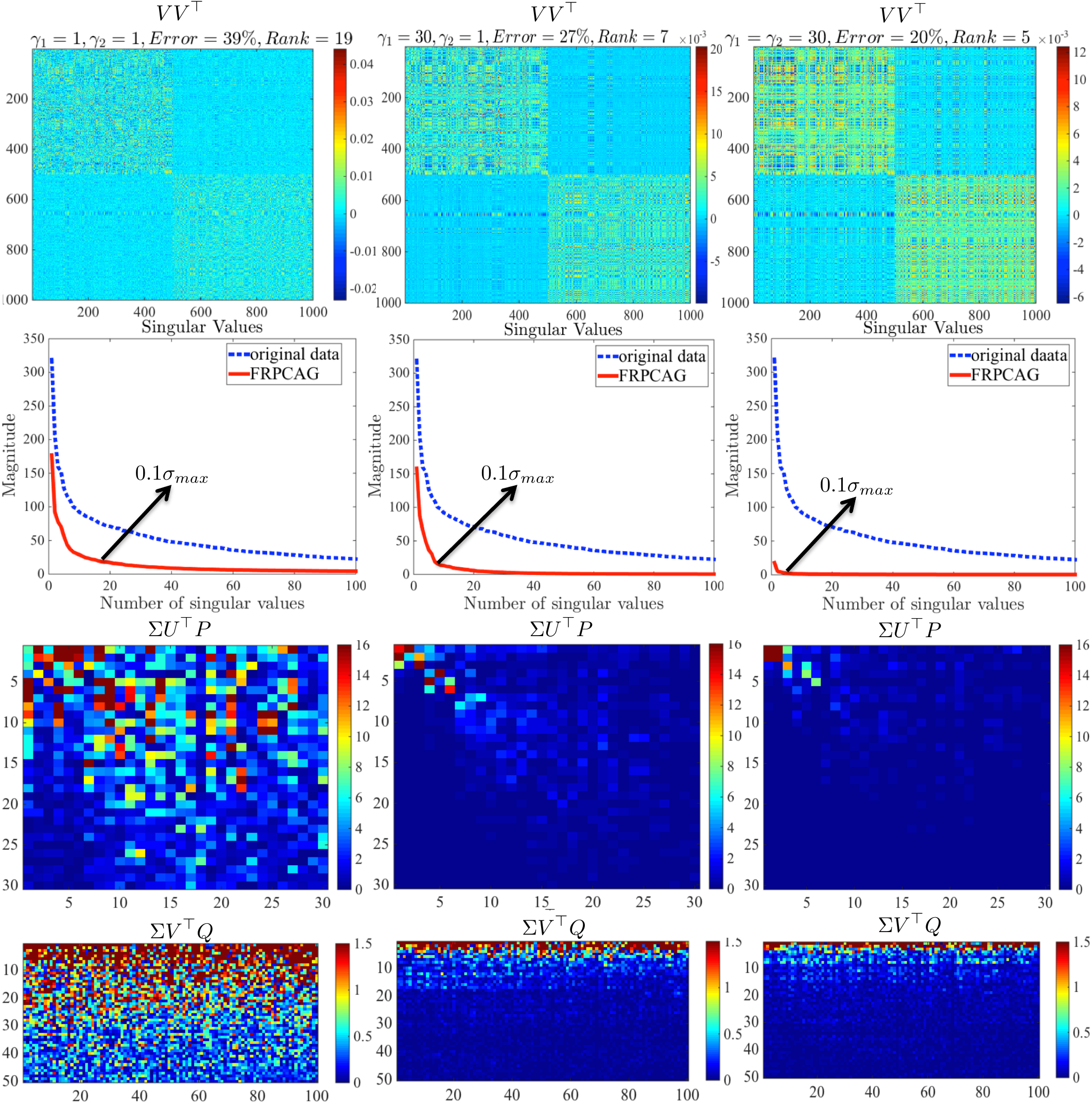}
         \caption{The matrices $VV^\top$, $\Sigma$, ${\Sigma}V^{\top}Q$, ${\Sigma}U^{\top}P$ and the corresponding clustering errors obtained for different values of the weights on the two graph regularization terms for 1000 samples of MNIST dataset (digits 0 and 1). If $X=U\Sigma V^{\top}$ is the SVD of $X$, then $V$ corresponds to the matrix of principal components (right singular vectors of $X$) and $U$ to the principal directions (left singular vectors of $X$). Let $\Larg_c = Q\Lambda_r Q^{\top}$ and $\Larg_c = P\Lambda_c P^{\top}$ be  the eigenvalue decompositions of $\Larg_r$ and $\Larg_c$ respectively then $Q$ and $P$ correspond to the eigenvectors of Laplacians $\Larg_c$ and $\Larg_r$.  The block diagonal structure of $VV^{\top}$ becomes more clear by increasing $\gamma_{r}$ and $\gamma_{c}$ with a thresholding of the singular values in $\Sigma$. Further, the sparse structures of ${\Sigma}V^{\top}Q$ and ${\Sigma}U^{\top}P$ towards the rightmost corners show that the  number of left and right singular vectors (weighted by singular values) which align with the eigenvectors of the Laplacians $\Larg_r$ and $\Larg_c$ go on decreasing with increasing $\gamma_r$ and $\gamma_c$. This shows that the two graphs help in attaining a low-rank structure with a low clustering error.}
        \label{fig:subspaces}
    \end{figure*}
    
\subsubsection{Experiment on the MNIST dataset}
We perform an experiment with 1000 samples of the MNIST dataset belonging to two different classes (digits 0 and 1). We vectorize all the digits and form a data matrix $Y$ whose columns contain the digits. Then we compute a graph of samples between the columns of $Y$ and a graph of features between the rows of $Y$ as mentioned in Section \ref{sec:graphs}. We determine the clean low-rank $X$ by solving model \eqref{eq:frpcag} and perform one SVD at the end $X = U\Sigma V^\top$. Finally, we do the  clustering  by performing k-means (k = 2) on the low-rank $X$. As argued in \cite{liu2013robust}, if the data is arranged according to the classes, the matrix $VV^\top $ (where $V$ are the principal components of the data) reveals the subspace structure. The matrix $VV^\top$ is also known as the shape interaction matrix (SIM) \cite{ji2015shape}.  If the subspaces are orthogonal then SIM  should acquire a block diagonal structure.  Furthermore,  our model \eqref{eq:frpcag} tends to align the first few principal components $v_i$ and principal directions $u_i$ of $X$ to the first few eigenvectors  $q_j$  and $p_j$ of $\Larg_c$ and $\Larg_r$ for larger $\sigma_i$ respectively. Thus, it is interesting to observe the matrices ${\Sigma}V^{\top}Q$ and ${\Sigma}U^{\top}P$ scaled with the singular values $\Sigma$ of the low-rank matrix $X$, as justified by eq. (\ref{eq:alignment}). This scaling takes into account the importance of the eigenvectors that are associated to bigger singular values.

Fig.~\ref{fig:subspaces} plots the matrix $VV^\top $, the corresponding clustering error, the matrices $\Sigma$, ${\Sigma}V^{\top}Q$, and ${\Sigma}U^{\top}P$  for different values of $\gamma_{r}$ and $\gamma_{c}$ from left to right. Increasing $\gamma_r$ and $\gamma_c$ from 1 to 30 leads to 1) the penalization of the singular values in $\Sigma$ resulting in a lower rank 2) alignment of the first few principal components $v_i$ and principal directions $u_i$ in the direction of the first few eigenvectors $q_j$ and $p_j$ of $\Larg_c$ and $\Larg_r$ respectively 3) an enhanced subspace structure in $VV^{\top}$ and 4) a lower clustering error. Together the two graphs help in acquiring a low-rank structure that is suitable for clustering applications as well. 

   \begin{figure*}[htbp]
    \centering
        \centering
        \includegraphics[width=1.0\textwidth]{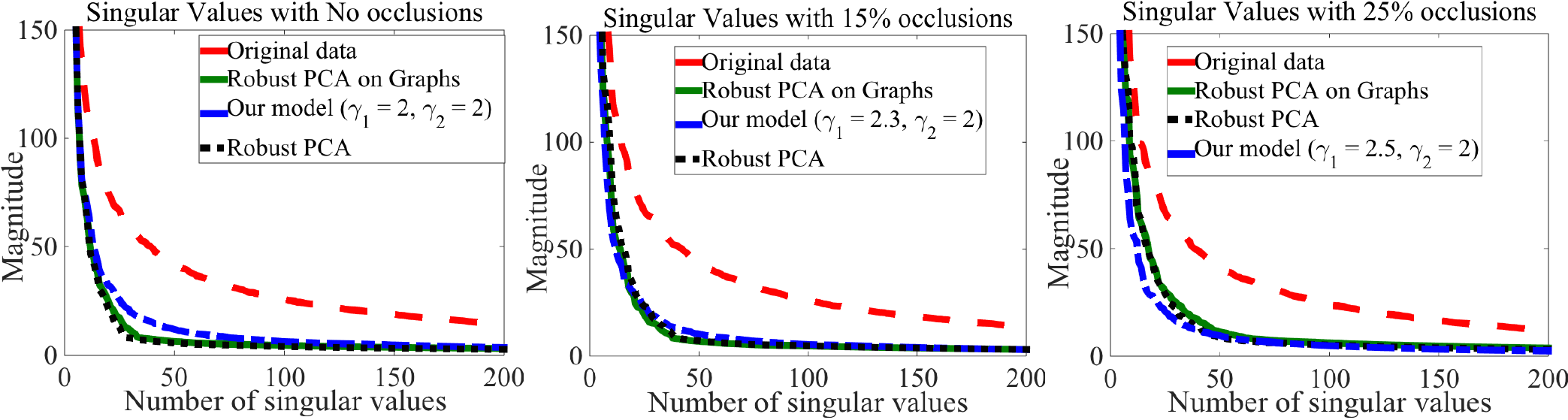}
     \caption{A comparison of singular values of the low-rank matrix obtained via our model, RPCA and RPCAG. The experiments were performed on the ORL dataset with different levels of block occlusions. The parameters corresponding to the minimum validation clustering error for each of the model were used.}
        \label{fig:svs_occlusions}
    \end{figure*}
    
 \subsubsection{Experiment on the ORL dataset with noise}   
    Next we demonstrate that for data with or without corruptions, FRPCAG is able to acquire singular values as good as the nuclear norm based models, RPCA and RPCAG. We perform three clustering experiments on 30 classes of ORL dataset with no block occlusions, 15\% block occlusions and 25\% block occlusions.  Fig.~\ref{fig:svs_occlusions} presents a comparison of the singular values of the original data with the singular values of the low-rank matrix obtained by solving RPCA, RPCAG and our model. The parameters for all the models are selected corresponding to the lowest clustering error for each model.  It is straightforward to conclude that the singular values of the low-rank representation using our fast method closely approximate those of the nuclear norm based models irrespective of the level of corruptions.

\subsection{Bi-clustering Example on ORL dataset}
We perform a small experiment with the 50 faces from  ORL dataset, corresponding to 5 different classes. We take the samples corresponding to the 5 different faces and pre-process the dataset to zero mean across the features and run FRPCAG with $\gamma_r = 3 = \gamma_c = 3$. Then, we cluster the low-rank $X$ to 5 clusters (number of classes in ORL) across the columns and 5 clusters across the features. While, the clustering across samples / columns has a straight-forward implication and has already been depicted in the previous subsection, we discuss the feature clustering in detail here. Fig. \ref{fig:biclusters} shows the 5 clusters of the features of the ORL dataset for one face of each of the 5 different types of faces. It is interesting to note that the clusters across the features correspond to very localized regions of the faces. For example, the first cluster (first row of Fig. \ref{fig:biclusters}) corresponds to the region of eyes and upper lips while the fifth cluster corresponds mostly to the smooth portion of the face (cheeks and chin). Of course, as FRPCAG is a low-rank recovery method which exploits the similarity information of the rows and columns, it also reveals the checkerboard pattern of the ORL dataset. This checkerboard pattern is visible in Fig. \ref{fig:checkerboard}. This pattern will become more visible for larger regularization parameters $\gamma_r$ and $\gamma_c$.

\begin{figure*}[htbp]
    \centering
        \centering
        \includegraphics[width=0.65\textwidth]{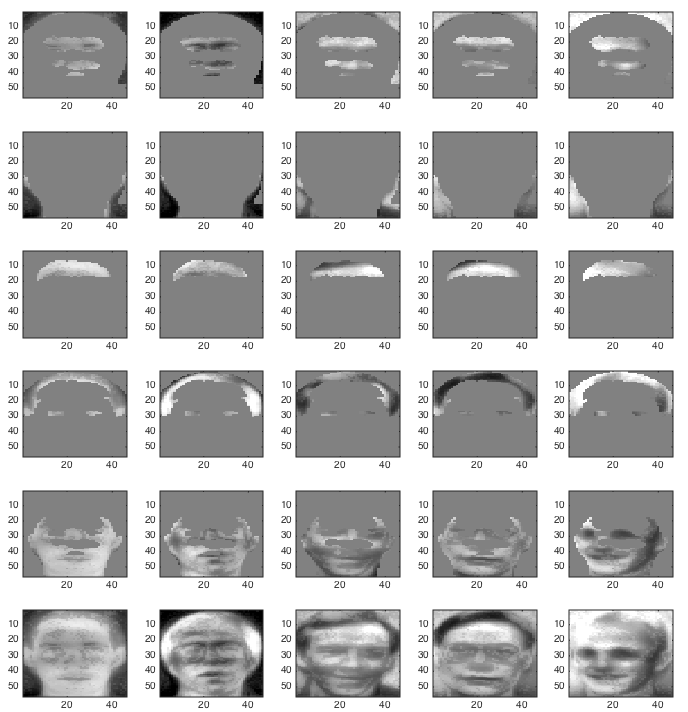}
         \caption{The clustering of features for ORL dataset. The top five rows show the five different feature clusters across five different faces. The last row shows the actual face which is a sum of all the feature clusters.}
        \label{fig:biclusters}
    \end{figure*}
    
    \begin{figure*}[htbp]
    \centering
        \centering
        \includegraphics[width=0.4\textwidth]{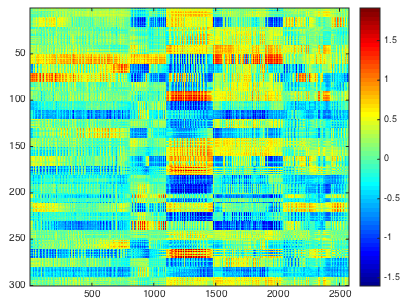}
         \caption{The checkerboard pattern of the low-rank representation of ORL dataset, obtained by FRPCAG.}
        \label{fig:checkerboard}
    \end{figure*}

\newpage
\section{Generalized Fast Robust PCA on Graphs (GFRPCAG)}\label{sec:gffrpcag}

\begin{mdframed}[style=MyFrame]
\begin{center}
\textbf{\textit{\nauman{``FRPCAG suffers from uncontrolled penalization of higher singular values by the lower graph eigenvalues, The Manifold Shrinking Effect. The approximation can be improved by deterministic graph filters which shape the graph spectra to favour reasonable penalization of data singular values''}}}. 
\end{center}
\end{mdframed}

In principle, model~\eqref{eq:frpcag}, which is a special case of~\eqref{eq:gfrpcag} has some connections with the state-of-the-art nuclear norm based low-rank recovery method RPCA \cite{candes2011robust}.  RPCA   determines a low-rank representation by:
\begin{enumerate}
\item Penalizing the singular values by a model parameter.
\item Determining the clean left and right singular vectors iteratively from the data by performing an SVD in every iteration.
\end{enumerate}

We first briefly discuss what makes our model an approximate recovery as compared to exact recovery, study what the approximation error is composed of and then use these arguments to  motivate a more general model. We say that our  model~\eqref{eq:frpcag} has connections with RPCA  because it performs the same two steps but in a different way, as explained in the previous section. In our model:
\begin{enumerate}
\item The singular values are penalized by the graph eigenvalues.  Therefore, one does not have a direct control on this operation as the thresholding depends on the graph spectrum.
\item The left and right singular vectors are constrained to be aligned with the low frequency eigenvectors of the row and column graph.  Thus, we take away the freedom of the singular vectors. In fact this is the primary reason our method does not need an SVD and scales very well for big datasets.
\end{enumerate}

\subsection{Re-visiting the Approximation Error of FRPCAG}
There are two main reasons of the approximation error, a glimpse of which has already been given above. We now discuss these two reasons in detail:

\begin{figure*}[htbp]
    \centering
        \centering
        \includegraphics[width=1.0\textwidth]{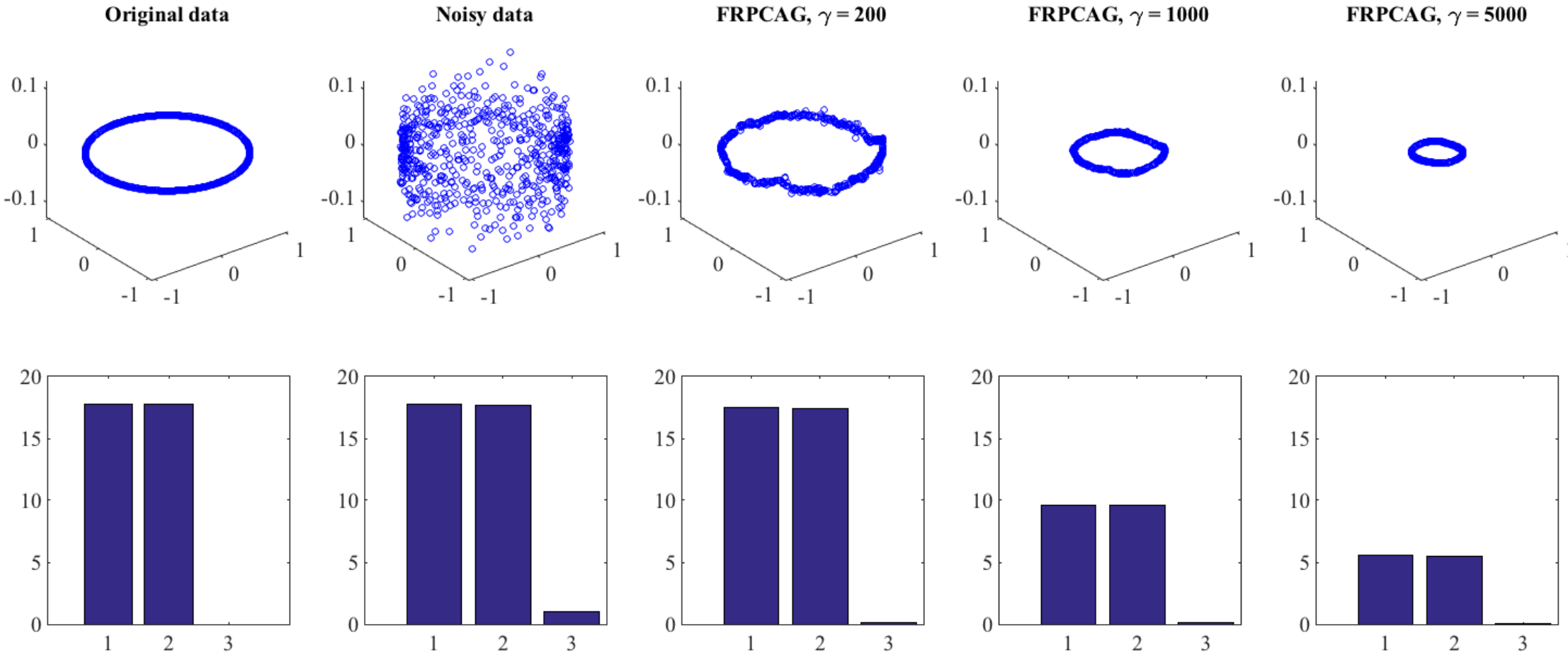}
         \caption{The manifold shrinking effect of a 2D manifold corrupted with noise. The first row shows the recovery of manifold with increasing values of FRPCAG (with one graph only) and the second row shows the singular values. Clearly, the $3^{rd}$ singular value does not become zero with higher penalization which degrades the lower singular values and the manifold begins to shrink. }
        \label{fig:shrinking}
    \end{figure*}

\subsubsection{Uncontrolled attenuation of the higher singular values: The Manifold Shrinking Effect} Ideally one would like the singular values corresponding to high frequency of the data matrix (smaller in magnitude) to be attenuated more because they correspond to noise and errors in the data. Furthermore, the  singular values corresponding to the lower frequencies, which constitute most of the data variance should be attenuated less. Such an operation would require the graph spectrum to be a step function where the position of the step depends on the rank of the matrix. For the readers who went through the theoretical details, from eq.~\ref{eq:bound} one can see that the error of the optimal low-rank $X^{*}$ depends on the two ratios: $\frac{\lambda_{k_c}}{\lambda_{k_{c}+1}}$ and $\frac{\lambda_{k_r}}{\lambda_{k_{r}+1}}$. These ratios are known as the spectral gaps of the graphs $\Larg_c$ and $\Larg_r$ respectively.  In practice, the graphs constructed via the $\mathcal{K}$-nearest neighbors strategy result in a smoothly increasing spectrum. Therefore, these ratios are never small enough. As a consequence, the undesired singular values (smaller in magnitude) are not attenuated enough and they are not exactly zero. If one forces higher attenuation by increasing the constants $\gamma_c$ and $\gamma_r$ then the desired singular values are penalized more than required. This results in a shrinking effect of the manifold as shown in Fig.~\ref{fig:shrinking}. In this Fig. the first row shows the recovery of manifold with increasing values of the graph regularization parameter in FRPCAG (with one graph only) and the second row shows the singular values. Clearly, the $3^{rd}$ singular value does not become zero with higher penalization which degrades the desired singular values and the manifold begins to shrink. This calls for a procedure to manually threshold the singular values beyond a certain penalization. 

\subsubsection{The subspace alignment gap} We do not allow the left and right singular vectors of the low-rank matrix to be determined freely via SVD. This results in a slight miss-alignment of the clean singular vectors of the resultant low-rank matrix  as compared to the singular vectors of the original clean data. Note that the alignment occurs as a result of a force acting on the singular vectors which depends on the regularization parameters. While increasing this force results in more alignment, the downside is the shrinking effect of singular values. Thus, a good parameter setting might still leave some alignment gap in order not to threshold the singular values more than required.

These two effects with the increasing values of the graph regularization parameter for FRPCAG have been illustrated in Fig.~\ref{fig:rotation_shrinking}. Here we show this effect for only one subspace (left singular vectors). Similar phenomenon takes place for the right singular vectors as well. For very small values of the regularization parameter, the singular vectors learned via FRPCAG are still aligned with the singular vectors of the noisy data and the thresholding on the higher singular values is not enough to remove noise. With the increasing regularization, the singular vectors start aligning and the higher singular values are thresholded more. Very high values of the regularization parameter might still provide a good alignment of the singular vectors but the singular values might be thresholded more than required, resulting in the shrinking effect. Thus, beyond a certain point, the increase of parameters contributes to higher errors due to the excessive shrinking of the singular values. 


\subsection{Design of deterministic graph filters `g' to Improve approximation error}

    The thresholding of the singular values is governed by the graph spectra as also pointed out in theorem~\ref{thm:gfrpcagnoisy}.  While the subspace rotation effect is an implicit attribute of our graph based low-rank recovery method and it cannot be improved trivially, we can do something to counter the shrinking effect to improve the approximation. A straight-forward method to improve the approximation would be to have a better control on the singular values by manipulating the graph spectra. To do so, it would be wise to define functions of the graph spectrum which make the spectral gap as large as possible. In simpler words we need a family of functions $g(a)$ which are characterized by some parameter $a$ that force the graph filtering operation to follow a step behavior.  More specifically, we would like to solve the following generalized problem instead of~\eqref{eq:frpcag1}.
    
\begin{align}\label{eq:gffrpcag}
& \min_{{X}} \Phi(X-Y) + \gamma_c\tr({X}g_c(\Larg_c){X}^\top) + \gamma_r\tr({X}^\top g_r(\Larg_r) {X})
\end{align}

The term $\tr({X}^\top g(\Larg){X}) = \| g(\Lambda)^\frac{1}{2} \hat{X}  \|$ allows for a better control of the graph spectral penalization.

 For our specific problem, we would like to have the lowest frequencies untouched as well as the highest frequencies strongly penalized. In order to do so, we propose the following design.

Let us first define a function
\[
h_b(x)=\begin{cases}
e^{-\frac{b}{x-\frac{b}{2}}} & \mbox{if }x\geq\frac{b}{2}\\
0 & \mbox{otherwise}.
\end{cases}
\]
Please note that this function is differentiable at all points. We then define $g_b$ as 
\[
g_b(x)=\frac{h_b(x)}{h_b(2b-x)}
\]
This new function has a few important properties. It is equal to $0$ for all values smaller than $\frac{b}{2}$, and then grows.  It is infinite for all values above $\frac{3}{2}b$. As a result, choosing the function $\sqrt{g_b}$ in our regularization term has the following effect: It preserves the frequency content for $\lambda_\ell << b$ and cuts the high frequencies $\lambda_\ell >>b $. The parameter $b$ acts as a frequency limit for our function. See Fig.~\ref{fig:filters}.

\subsection{Inner mechanism of the proposed regularization}
Since we solve our problem with proximal splitting algorithm, we also need to compute the proximal operator of $\tr({X}^\top g(\Larg){X}) = \| g_b(\Lambda)^\frac{1}{2} \hat{X}  \|$.
The computation is straightforward and explains the inner mechanism of such a regularization term. We first compute the gradient of 
\[
\argmin_{x}\|x-y\|_{2}^{2}+\gamma\|g_b(\Lambda)^{\frac{1}{2}}Q^{\top}x\|_{2}^{2}
\]
and set is to $0$.
\[
2(x-y)+ 2\gamma Qg_b(\Lambda)Q^{\top}x=0
\]
We then rewrite the following expression in the spectral domain:
\[
\hat{x}(\ell)-\hat{y}(\ell)+\gamma g_b(\lambda_{\ell})\hat{x}(\ell)=0,\hspace{1em}\forall\ell\in\{0,\dots,n-1\}
\]
Shuffling the equation leads to the solution.
\begin{equation} \label{eq:spectral-filtering}
\hat{x}(\ell)=\frac{1}{1+\gamma g_b(\lambda_{\ell})}\hat{y}(\ell),\hspace{1em}\forall\ell\in\{0,\dots,n-1\}
\end{equation}
Equation~\ref{eq:spectral-filtering} shows that the proximal operator of $\tr({X}^\top g(\Larg){X}) $ is a graph spectral filtering with
\begin{eqnarray*}
f_b(x,\gamma) & = & \frac{1}{1+\gamma g(x,b)}\\
 & = & \frac{h(2b-x)}{h(2b-x)+\gamma h(x)}
\end{eqnarray*}
In fact the filter $f_b$ is infinitely many times differentiable and is smooth enough to be be easily approximated by a low order polynomial. We observe that $f_b(\cdot,1)$ is a low pass filter with bandwidth $b$. The computation cost of a filtering operation grows with the number edges $e$ and the order of the polynomial $o$: $\mathcal{O}(oe)$~\cite{susnjara2015accelerated}. In figure~\ref{fig:filters}, we show and example of the different aforementioned functions.

\begin{figure}[htbp]
    \centering
        \centering
        \includegraphics[width=0.6\textwidth]{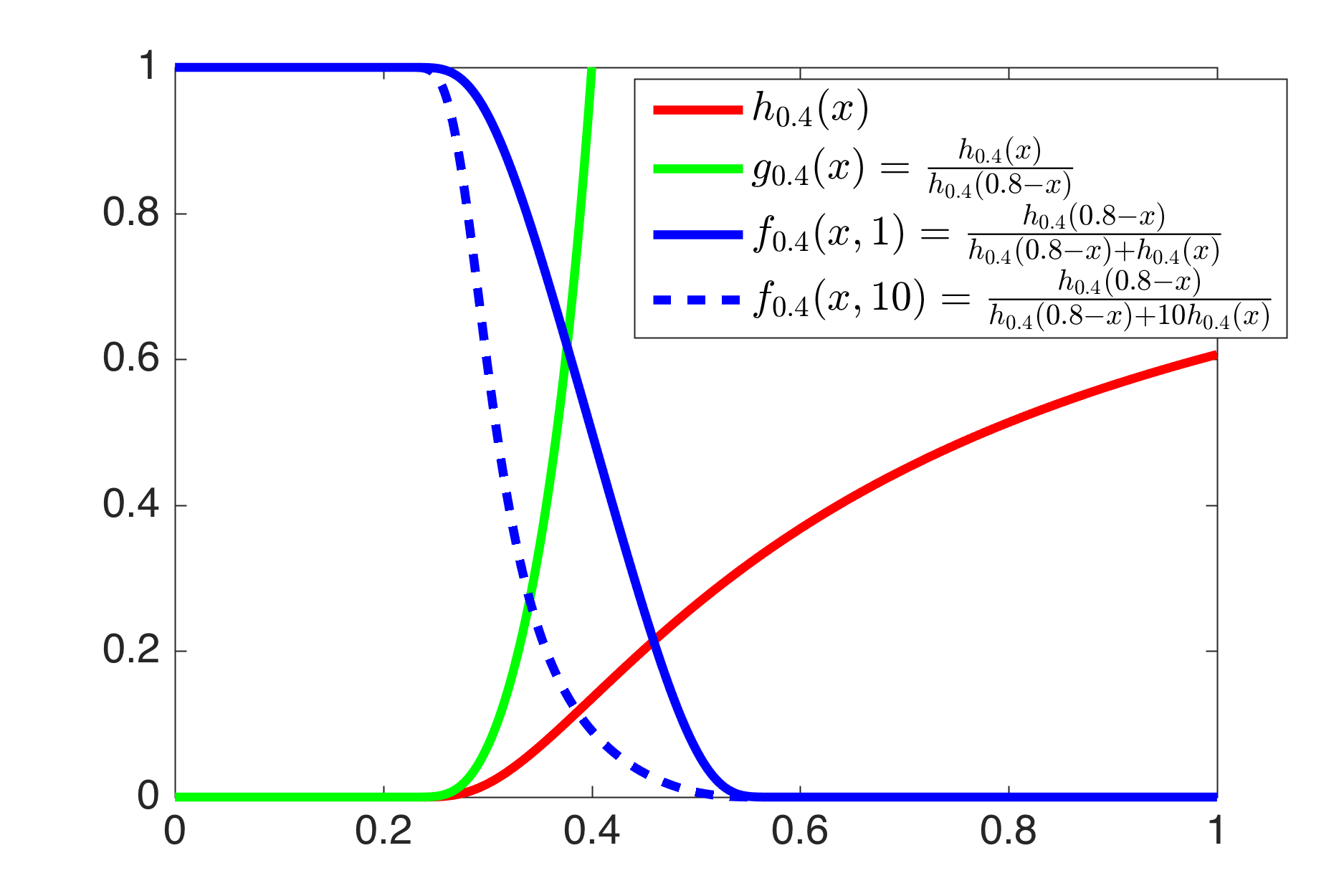}
         \caption{Example of functions. Here $b=0.4$.}
        \label{fig:filters}
    \end{figure}


\subsection{Algorithm for GFRPCAG}
In order to solve our problem, we use  a forward backward based primal dual for GFRPCAG~\cite{komodakis2014playing} which can be cast into the following form
\begin{equation}
\argmin_x a(x) + b(x) + c(x)
\end{equation}
Our goal is to minimize the sum $a(x) + b(x) + c(x)$, which is done efficiently with proximal splitting methods \cite{combettes2011proximal,komodakis2014playing}. We set $a(X) = \gamma_{c}\tr(X^\top \Larg_{c}X)$, $b(x) = \gamma_{r}\tr(X g(\Larg_{r})X^\top )$ and $c(X) = \|Y-X\|_{1}$ and $c(X)=0$. The gradient of $a$ becomes $\nabla_a(X) =  2 \gamma_{r} X\Larg_{r}$. \textit{Note that here we assume the presence of a deterministic filter only on the graph of columns}. 
We define an upper bound on the Lipschitz constant $\beta$ as $\beta \leq \beta' = 2\gamma_r \|\Larg_r\|_2$. The proximal operator of the $\ell_1$ soft-thresholding given by the elementwise operations,

\begin{equation}\label{eq:prox}
\prox_{\lambda b }(X) = Y + \sign(X-Y) \circ \max (|X-Y|-\lambda ,0),
\end{equation}
 here $\circ$ is the Hadamard product. Finally, from \eqref{eq:spectral-filtering} the proximal operator of the function $b$ is given by a graph spectral filtering.

The forward-backward primal~\cite{komodakis2014playing} Algorithm \ref{CHalgorithm2} can now be stated, where the time-steps $\tau_1 = \frac{1}{\beta}$, $\tau_2 = \frac{\beta}{2}$, $\tau_3 = 0.99$, $\epsilon$ the stopping tolerance and $J$ the maximum number of iterations. $\delta$ is a very small number to avoid a possible division by $0$.
\begin{algorithm}
\caption{Forward-backward primal dual for GFRPCAG with filter on column graph only.}
\label{CHalgorithm2}
\begin{algorithmic}
\State INPUT: $X_0 = Y$, $V_0 = Y$, $\epsilon > 0$
\For{ $j = 0,\dots J-1$ }
\State $P_{j} = \prox_{\tau_1 f}\left(X_{j} - \tau_1 \left(\nabla_h(X_{j}) +  V_j \right) \right)$
\State $T_j = V_j + \tau_2 (2P_j-X_j)$
\State $Q_{j} = T_j - \tau_2 \prox_{\frac{1}{\tau_2}g} \left( \frac{1}{\tau_2} T_j \right) $
\State $(X_{j+1},V_{j+1}) = (U_{j},V_{j}) + \tau_3 \left(  (P_{j},Q_{j}) -  (X_{j},V_{j})  \right)$
\If{$\frac{\|X_{j+1} - X_{j}\|_F^2}{\| U_{j}\|_F^2+\delta}<\epsilon$ and $\frac{\|V_{j+1} - V_{j}\|_F^2}{\| V_{j}\|_F^2+\delta}<\epsilon$}
\State BREAK\textbf{}
\EndIf
\EndFor
\end{algorithmic}
\end{algorithm}

\section{When does the dual-graph filtering makes sense?}\label{sec:when}
In the final Section of this extensive work we discuss some important implications of our work and its connections with the state-of-the-art that uses dual graph based filtering. Obviously, our dual graph filtering approach makes sense for a number of real world applications mentioned throughout the whole paper. This is proved experimentally in the results section of this work as well. Apparently, the use of dual-graph filtering should help in a broad range of clustering applications. However, we highly recommend that this approach cannot be adopted as a rule-of-thumb in any graph based convex optimization problem, specially if one wants to target an exact low-rank recovery. Interestingly, we point out a few examples from the state-of-the-art which have close connections with our approach and justify why our model is different from those works.

\subsection{Collaborative filtering: Recommendation Systems} The collaborative filtering approach is commonly used in the recommendation systems where one uses a notion of graph based similarity between the users and items to perform an effective recommendation. The authors of \cite{gu2010collaborative} use an NMF based approach with dual graph filtering in their proposed recommendation system. Similarly, the authors of \cite{kalofolias2014matrix} propose a hybrid recommendation with nuclear-norm based exact recovery and dual-graph filtering. While the authors of \cite{kalofolias2014matrix} shown that when the number of measurements is very small then their proposed dual-graph filtering approach outperforms the state-of-the-art methods, their method does not improve upon the exact recovery of the nuclear-norm based approach \cite{candes2010matrix}. Thus, it makes sense to use the dual-graph filtering to reduce the error of approximation in the low-rank but it does not ensure an exact recovery if a standard exact recovery mechanism (like nuclear norm) is not available.

\subsection{Robust PCA on Graphs} Our previous work \cite{shahid2015robust}, which uses a nuclear norm based exact recovery framework with graph filtering has also proved that the use of graph filtering does not improve upon the exact recovery conditions over the standard RPCA framework \cite{candes2011robust}. In fact, as shown in the rank-sparsity plots in  \cite{shahid2015robust}, the graph based filtering reduces the error over RPCA \cite{candes2011robust} for various rank-sparsity pairs but it does not guarantee an exact recovery in the region where the nuclear norm based recovery fails.

The conclusions drawn from the above two state-of-the-art techniques are exactly in accordance with the motivation of this work. Our motivation was not to design a framework that returns an exact low-rank representation but to present a system that works well under a broad range of conditions in contrast to those required for exact recovery \cite{candes2011robust}. Thus, we target an approximate and fast recovery under a broad range of conditions and data types. This is exactly why our proposed framework, apparently very closely related to \cite{gu2010collaborative} \& \cite{kalofolias2014matrix}, is entirely different from these works. The recovery of our method totally depends on the quality of information provided by the graphs.

Moreover, our framework relies on the assumption that the data matrices are dual band-limited on two graphs. Although, such an assumption is implicit in \cite{gu2010collaborative} \& \cite{kalofolias2014matrix}, the authors do not discuss it theoretically or experimentally in their works. Furthermore, RPCAG  \cite{shahid2015robust} uses only one graph. Thus, the assumption behind this model is totally different from our proposed framework.

\bibliographystyle{IEEEtran}
\bibliography{pcabib.bib}
\end{document}